\documentclass[letterpaper, 10 pt, conference]{ieeeconf}

\IEEEoverridecommandlockouts
\usepackage[T1]{fontenc}
\usepackage{amsmath,amsfonts}
\usepackage{algorithm, algpseudocode}
\usepackage{graphicx}
\usepackage{booktabs}
\usepackage{microtype}

\let\labelindent\relax
\usepackage{enumitem}

\usepackage{tikz}
\usetikzlibrary{arrows.meta,positioning,calc,decorations.pathreplacing}
\usetikzlibrary{shapes.misc,shapes.geometric,fit,intersections}

\newtheorem{theorem}{Theorem}
\newtheorem{remark}{Remark}
\newtheorem{proposition}{Proposition}

\title{\LARGE \bf
    A MILP-Based Solution to Multi-Agent Motion Planning \\ and Collision Avoidance in Constrained Environments
}

\author{Akshay Jaitly, Jack Cline, and Siavash Farzan%
\thanks{Akshay Jaitly is with the Robotics Engineering Department, Worcester Polytechnic Institute, 
Worcester, MA 01609, USA.}%
\thanks{Jack Cline and Siavash Farzan are with the Electrical Engineering Department, California Polytechnic State University, San Luis Obispo, CA 93407, USA. (Corresponding email: sfarzan@calpoly.edu)}}%

\begin{document}

\maketitle
\thispagestyle{empty}
\pagestyle{empty}

\begin{abstract}
We propose a mixed-integer linear program (MILP) for multi-agent motion planning that embeds Polytopic Action-based Motion Planning (PAAMP) into a sequence-then-solve pipeline. Region sequences confine each agent to adjacent convex polytopes, while a big-M hyperplane model enforces inter-agent separation. Collision constraints are applied only to agents sharing or neighboring a region, which reduces binary variables exponentially compared with naive formulations. An L1 path-length-plus-acceleration cost yields smooth trajectories. We prove finite-time convergence and demonstrate on representative multi-agent scenarios with obstacles that our formulation produces collision-free trajectories an order of magnitude faster than an unstructured MILP baseline.
\end{abstract}

\section{Introduction}

Multi-agent motion planning is a fundamental challenge in robotics, particularly in applications such as warehouse logistics (where a fleet of autonomous forklifts navigating tight, human-occupied aisles~\cite{electronics12183820}), as well as in collaborative manufacturing, autonomous driving, and UAV swarms~\cite{kagan2019autonomous,multi2023}.  
The primary objective is to generate dynamically feasible trajectories while ensuring collision avoidance among agents and obstacles.
Conventional approaches, including decentralized control, sampling-based methods, and nonlinear optimization, often suffer from high computational complexity, limited scalability, or infeasibility in real-world settings~\cite{RRT, PRM, direct-collocation}.

Over the past decade, various methodologies have been proposed to address these challenges. Sampling-based algorithms, such as Rapidly-exploring Random Trees (RRT) and Probabilistic Roadmaps (PRM), are widely used for their ability to handle high-dimensional state spaces~\cite{MultiRRT2013, Yang2008, Augugliaro2012}. Despite their probabilistic completeness, these methods struggle in constrained environments with narrow passages and multi-agent coordination. Alternatively, direct transcription and collocation techniques convert continuous trajectory planning into large-scale nonlinear programs~\cite{Turnbull2013,Robinson2018}, but these approaches often face issues with local minima and computational scalability.

Mixed-Integer Linear Programming (MILP) has gained traction in motion planning due to its ability to incorporate discrete decisions and enforce hard constraints while providing global feasibility guarantees~\cite{MILP_book}. Traditional MILP-based methods use big-$M$ formulations to encode disjunctive constraints, such as collision avoidance, but their applicability remains limited by the combinatorial complexity of large-scale problems.

Collision avoidance is traditionally addressed through nonlinear continuous methods, such as GJK~\cite{GJK} and DCOL~\cite{Tracy2023}, which employ convex optimization to refine motion plans. The approach in~\cite{LeCleac2023} formulates collision avoidance and trajectory optimization as a non-smooth, single-level optimization problem using nonlinear complementarity constraints.
However, such methods lead to computationally demanding, non-convex problems and lack completeness, meaning feasible solutions are not guaranteed to be found.

Region-based planning approaches offer an alternative means of simplifying multi-agent coordination. Methods like Safe Interval Path Planning (SIPP)~\cite{SIPP} and Conflict-Based Search (CBS)~\cite{sharon2015conflict} decompose the problem into spatial regions and temporal intervals. Although these techniques improve scalability, they often produce trajectories that lack smoothness or dynamic feasibility.

The recently proposed Polytopic Action-Set and Motion Planning (PAAMP) framework~\cite{Jaitly2024} addresses this limitation by representing feasible trajectory sets as unions of polytopic subsets. PAAMP employs a heuristic to identify relevant polytopic subsets (a discrete optimization step) and a linear program to find solutions within them. If no valid continuous solution exists, the discrete solution is refined iteratively in a `sequence-then-solve' approach~\cite{TAMPsurvey}. While PAAMP improves planning efficiency with dynamic constraints, it does not fully address multi-agent collision avoidance in a computationally efficient manner.

Similarly, Graphs-of-Convex-Sets (GCS) planning~\cite{GCS} partitions the search space into convex polytopes and optimizes over them.  
For single-agent systems (e.g., quadcopters flying through clutter) GCS has been shown to handle non-convex obstacle avoidance effectively by adding a modest number of binary variables~\cite{GCS_state_space}.  
In multi-agent settings, however, every additional robot introduces $O(N^2)$ pairwise region–pair checks, so a direct GCS expands the mixed-integer program and quickly becomes intractable.  
Furthermore, the baseline GCS formulation omits dynamics and requires extra transcription steps to enforce system constraints.

The key challenge in applying PAAMP to multi-agent collision avoidance lies in the nonlinear nature of collision constraints. Verifying solution existence within a polytopic trajectory set remains computationally expensive and potentially incomplete. This work presents a novel integration of PAAMP principles with an efficient MILP formulation: partitioning the workspace into polytopic regions, determining agent sequences through these regions, and solving a MILP to generate optimal trajectories. To maintain computational tractability, we employ separating hyperplanes with a big-$M$ formulation and ensure proper agent separation while avoiding complex nonlinear constraints. By selecting candidate separating directions and enforcing that at least one maintains sufficient separation, we achieve effective collision avoidance. Additionally, incorporating constraints for obstacle avoidance and trajectory smoothing allows us to generate high-quality, collision-free trajectories for multiple agents in constrained environments while significantly reducing the combinatorial complexity of multi-agent planning.

The main contributions of this paper are as follows:
\begin{list}{}{\leftmargin=0em \itemindent=6pt}
    \item[i.] Formulating the multi-agent trajectory planning problem as a MILP that integrates region-based sequence planning with an efficient collision avoidance formulation to handle inter-agent coordination.
    \item[ii.] Developing a separating hyperplane approach for collision avoidance that effectively certifies the existence of a solution, and handles safety constraints while maintaining computational tractability.
    \item[iii.] Incorporating acceleration penalties to generate smooth, natural trajectories that respect environmental constraints.
    \item[iv.] Providing theoretical guarantees on the convergence and optimality, with a rigorous analysis of computational complexity reduction through
    relevant pair identification.
    \item[v.] Demonstrating through comprehensive simulations that the PAAMP-based MILP framework generates feasible, collision-free trajectories for multiple agents in challenging environments with reduced computation time compared to naive approaches.
\end{list}

\section{MILP-Based Multi-Agent Trajectory Planning}\label{sec:milp}

In this section, we present our integrated approach to multi-agent trajectory planning, which combines the expressiveness of Mixed-Integer Linear Programming (MILP) with the efficiency of Polytopic Action-based Motion Planning (PAAMP). We consider a team of $N$ agents that must cooperatively navigate through constrained environments while avoiding collisions. Our approach decouples the combinatorial complexity of multi-agent planning through a sequence-then-solve framework that maintains safety and dynamic feasibility guarantees.

\subsection{Problem Formulation}
Let $\mathcal X_i\subset\mathbb R^{r_i}$ be the state space of agent~$i$ and $\mathcal O\subset\mathbb R^{r}$ the union of convex obstacles. Here \(r_i\) is the dimension of agent \(i\)’s configuration space, and we set \(r := \max_i r_i\) as the common dimension of the ambient Euclidean workspace (\(r=2\) in our planar examples).

Given initial/goal states $x_{0,i},x_{f,i}\in\mathcal X_i$, find piece-wise linear trajectories $\{x_i^{(k)}\}_{k=0}^{T}$ that minimize a cost $J=\sum_i J_i$, respect per-segment velocity bounds, avoid $\mathcal O$, and maintain a minimum inter-agent distance $d_{\min}$. Time is discretized into $T$ equal steps.

\subsection{PAAMP Region Representation}\label{subsec:region-rep}
In the proposed PAAMP-based approach, the workspace is discretized into a finite set of polytopic regions $\mathcal{R} = \{R_1, R_2, \ldots, R_M\}$, each given by
\begin{equation}\label{eq:region_definition}
    R_j = \{x \in \mathbb{R}^{r} \;|\; A_j x \leq b_j\},
\end{equation}
where $x_i(t) \in \mathbb{R}^{r_i}$ represents the state of agent $i$, with $i = 1,\ldots,N$, at time $t$. Matrices $A_j\in\mathbb{R}^{n_j\times r}$ and $b_j\in\mathbb{R}^{n_j}$ bound the regions, whose union approximates the workspace.
The integer \(n_j\) denotes the number of facet-defining
half-spaces of region \(R_j\) (i.e.\ the rows of \(A_j\)).

We construct a region–adjacency graph $G_R = (V_R, E_R)$, where each vertex $v_j \in V_R$ corresponds to a region $R_j$, and an edge $(v_j, v_k) \in E_R$ exists if and only if regions $R_j$ and $R_k$ are geometrically adjacent or overlapping. This graph is utilized in the sequence generation phase to ensure that agents can only transition between adjacent regions.
Adjacency is determined by checking if $R_j \cap R_k$ is non-empty.

We constrain each trajectory segment $k$ (linear interpolation between some $x^{k}_i$ and $x^{k+1}_i$) to remain within a specific $R_j$ by restricting $x^{k}_i \in R_j$, $x^{k+1}_i \in R_j$.

\subsection{Collision Avoidance Between Agents}\label{subsec:collision}
To ensure safety, agents must maintain a minimum separation distance $d_{\min}$. Let $x_i^{(k)}$ denote the state of agent $i$ at a discrete time $t_k$ (for $k=0,1,\ldots,T$). The desired collision avoidance condition is
\begin{equation}\label{eq:collision_orig}
    \|x_i^{(k)} - x_j^{(k)}\|_2 \geq d_{\min}, \quad \forall\, i<j,\; \forall\, k.
\end{equation}

Since the Euclidean norm constraint is nonconvex, we adopt a disjunctive formulation using separating hyperplanes. For each candidate separating direction indexed by $l$, we introduce binary variables $\delta_{i,j}^{(k,l)}$ to indicate whether that direction is used for separation:
\begin{equation}\label{eq:bigM_constraint}
    \langle c_{i,j}^{(l)}, x_j^{(k)} - x_i^{(k)} \rangle \geq d^{(l)} - M(1-\delta_{i,j}^{(k,l)}), \; \forall\, l=1,\ldots,L,
\end{equation}
where each vector $c_{i,j}^{(l)}\in\mathbb{R}^{r}$ is a unit vector (i.e., $\|c_{i,j}^{(l)}\| = 1$) representing a candidate separating hyperplane, $d^{(l)}>0$ is a specified threshold along the $l$-th direction, and $M$ is a sufficiently large constant.

We pre-compute \(L\) unit vectors
\(\{c_{i,j}^{(l)}\}_{l=1}^{L}\subset\mathbb R^{r}\) that uniformly
sample the unit sphere and choose a single big-$M$ constant
\(M\!\gg\! d_{\min}\) to deactivate an inequality whenever its
binary indicator \(\delta_{i,j}^{(k,l)}\) equals 0.

The requirement that at least one candidate separating hyperplane is active is enforced by the logical condition:
\begin{equation}\label{eq:binary_condition}
    \sum_{l=1}^{L}\delta_{i,j}^{(k,l)} \geq 1, \quad \forall\, i<j,\; \forall\, k.
\end{equation}

This approach allows us to reformulate the nonconvex collision avoidance constraints as a set of mixed-integer linear constraints through the use of the big-$M$ technique, which has been widely applied in mixed-integer programming.

By enforcing separation along at least one of several candidate directions, we provide a piecewise linear approximation of the nonconvex separation constraint while preserving the MILP structure.
The formulation integrates seamlessly with the sequence-based planning approach, where we only need to enforce collision avoidance constraints between relevant agent pairs.
Additionally, by carefully selecting the candidate separating directions, we can balance computational efficiency with the accuracy of the collision avoidance approximation. The use of binary variables to select active separating hyperplanes provides a clear decision-making structure that aligns with the discrete nature of the sequence-based planning approach.

\subsection{Sequence-Then-Solve Approach with PAAMP}\label{subsec:sequence}
To efficiently handle the inherent combinatorial complexity of multi-agent trajectory planning, we adopt a sequence-then-solve approach inspired by the PAAMP framework. This approach decouples the planning process into two distinct stages: (1) sequence generation, where candidate region sequences are determined for each agent, and (2) lower-level optimization, where a MILP solver finds optimal trajectories while respecting these sequences and ensuring collision avoidance.

\subsubsection{Region-Based Sequence Generation}
Given the previously defined region–adjacency graph \(G_R\)
(see Section~\ref{subsec:region-rep}) built from
\(\mathcal R\), we enumerate up to \(K\) candidate sequences
\(\pi_i=\{\pi_{i,1},\dots,\pi_{i,T}\}\) for each agent by running a
cost-weighted shortest-path search from the region
containing \(x_{0,i}\) to that containing \(x_{f,i}\).
A sequence is admissible if and only if
\(x_{0,i}\in R_{\pi_{i,1}}\),
\(x_{f,i}\in R_{\pi_{i,T}}\), and
\((\pi_{i,t},\pi_{i,t+1})\in E_R\).

The first two conditions ensure that the initial and goal states are contained in the first and last regions of the sequence, respectively. The third condition ensures that consecutive regions in the sequence are adjacent, which allows the agent to transition between them.

If the lower-level MILP fails, we blacklist the conflicting
time-region edge in $G_R$ and re-plan a new \(\pi_i\); this ``refine-on-conflict''
loop repeats until a feasible joint sequence is found or a timeout
occurs.

\subsubsection{Relevant Pair Identification}
A key efficiency improvement in our approach comes from identifying which pairs of agents need collision avoidance constraints at each time step. Given sequences $\pi_i$ and $\pi_j$ for agents $i$ and $j$, we define a pair $(i,j)$ as \textit{relevant} at time $t$ if their assigned regions are the same or adjacent:
\begin{equation}\label{eq:relevant_pair}
    \text{relevant}_{i,j}^{(t)} = \begin{cases}
    1, & \text{if } \pi_{i,t} = \pi_{j,t} \text{ or } (\pi_{i,t}, \pi_{j,t}) \in E_R \\
    0, & \text{otherwise}
    \end{cases}
\end{equation}
For each time step $t$, we compute the set of all relevant pairs:
\begin{equation}\label{eq:relevant_pairs_set}
    \mathcal{P}^{(t)} = \{(i,j) \;\mid\; i < j \;\text{ and }\; \text{relevant}_{i,j}^{(t)} = 1\}
\end{equation}

The collision avoidance constraints in the MILP formulation are then only applied to agent pairs in $\mathcal{P}^{(t)}$ at time $t$, which substantially reduces the number of binary variables in the optimization problem.

\subsubsection{Lower Level Optimization}
Given a complete set of sequences $\boldsymbol{\pi} = \{\pi_1, \pi_2, \ldots, \pi_N\}$ for all agents, we must determine if this joint sequence is admissible, i.e. there exist feasible, collision-free trajectories that follow these sequences.
For each agent we penalize path length and acceleration with the per-agent cost:

\noindent
\begin{equation}\label{eq:agent_cost}
\begin{aligned}
    J_i(\{x_i^{(k)}\}) \;=\;
    & \sum_{k=0}^{T-1}\!\|x_i^{(k+1)}-x_i^{(k)}\|_1 \\
    & +\alpha\!\sum_{k=1}^{T-2}\!\|x_i^{(k+1)}-2x_i^{(k)}+x_i^{(k-1)}\|_1,
\end{aligned}
\end{equation}
This objective function balances two key aspects, where the first term minimizes the total distance traveled, and the second term, weighted by $\alpha$, promotes smooth trajectories by penalizing acceleration. The $L_1$ norm is used for both terms and is implemented using auxiliary variables and linear constraints.
This balanced objective helps avoid unnecessary detours while ensuring smooth motion profiles.

A joint sequence $\boldsymbol{\pi}$ is admissible if and only if the corresponding MILP in~\eqref{eq:multi_agent_MILP_PAAMP} has a feasible solution.
\begin{equation}\label{eq:multi_agent_MILP_PAAMP}
\begin{aligned}
    &\min_{\{x_i^{(k)}\},\{\delta_{i,j}^{(k,l)},\,\gamma_{i,p}^{(k,q)}\}} \quad \sum_{i=1}^{N} J_i(\{x_i^{(k)}\}) \\
    \text{s.t.}\quad 
    &\textit{(Boundary Conditions)}\\
    &\quad x_i^{(0)} = x_{0,i}, \quad x_i^{(T)} = x_{f,i},\quad \\ 
    &\hspace{1.5in}\forall i=1,\ldots,N,\\
    &\textit{(Dynamic Feasibility)}\\
    &\quad \|x_i^{(k+1)} - x_i^{(k)}\|_{\infty} \leq v_{\max},\quad \\
    &\hspace{1in}\forall i=1,\ldots,N,\; k=0,\ldots,T-1,\\
    &\textit{(Region Membership)}\\
    &\quad A_j x_i^{(k)} \leq b_j,\\
    &\quad A_j x_i^{(k{+}1)} \leq b_j,\;\; \text{if agent $i$ is in $R_j$ at segment $k$}, \\
    &\hspace{1.1in}\forall i=1,\ldots,N,\; k=0,\ldots,T-1,\\
    &\textit{(Collision Avoidance)}\\
    &\quad \langle c_{i,j}^{(l)}, x_j^{(k)} - x_i^{(k)} \rangle \geq d^{(l)} - M(1-\delta_{i,j}^{(k,l)}),\quad \\
    &\hspace{0.5in}\forall\, i,j \text{ with } 1\le i<j\le N,\; \forall l=1,\ldots,L,\\
    &\quad \sum_{l=1}^{L}\delta_{i,j}^{(k,l)} \geq 1,\quad \\
    &\hspace{0.5in}\forall\, i,j \text{ with } 1\le i<j\le N,\\
    &\quad \delta_{i,j}^{(k,l)} \in \{0,1\},\quad \\
    &\hspace{0.5in}\forall\, i,j \text{ with } 1\le i<j\le N,\; \forall l=1,\ldots,L,\\
    &\textit{(Obstacle Avoidance)}\\
    &\quad (A_p^{\text{obs}})_q x_i^{(k)} \geq (b_p^{\text{obs}})_q + \epsilon - M(1-\gamma_{i,p}^{(k,q)}),\quad \\
    &\hspace{0.2in}\forall i=1,\ldots,N,\; \forall p=1,\ldots,P,\; \forall k=0,\ldots,T,\; \forall q,\\
    &\quad \sum_{q} \gamma_{i,p}^{(k,q)} \geq 1,\quad \\
    &\hspace{0.2in}\forall i=1,\ldots,N,\; \forall p=1,\ldots,P,\; \forall k=0,\ldots,T,\\
    &\quad \gamma_{i,p}^{(k,q)} \in \{0,1\},\quad \\
    &\hspace{0.2in}\forall i=1,\ldots,N,\; \forall p=1,\ldots,P,\; \forall k=0,\ldots,T,\; \forall q.
\end{aligned}
\end{equation}
where \(\gamma_{i,p}^{(k,q)}\) are binary selector variables that choose
the active half‑space for obstacle~\(p\) at time~\(k\).

Note that for the region membership
constraint, if $x_i^{(k)}\in R_{j}$ and $x_i^{(k+1)}\in R_{j'}$ with
$j\neq j'$, we additionally require
$(v_j,v_{j'})\in E_R$ and enforce both
$A_{j}x_i^{(k)}\!\le\! b_{j}$, $A_{j'}x_i^{(k+1)}\!\le\! b_{j'}$.

$\epsilon>0$ in~\eqref{eq:multi_agent_MILP_PAAMP} is a user-specified
safety margin (typically $1{-}5$ cm) that guarantees a buffer between an
agent and the obstacle half-spaces $A^{\text{obs}}_p x\ge b^{\text{obs}}_p$

If this problem is feasible, then the sequence is admissible, and the solution is a locally optimal trajectory. If not, we must generate a new set of sequences.

For the acceleration term, we introduce auxiliary variables $a_{i,x}^{(k)}$ and $a_{i,y}^{(k)}$ for the $x$ and $y$ components of acceleration, and constrain them as follows:
\begin{align}
a_{i,x}^{(k)} &\geq (x_i^{(k+1)})_x - 2(x_i^{(k)})_x + (x_i^{(k-1)})_x \\
a_{i,x}^{(k)} &\geq -((x_i^{(k+1)})_x - 2(x_i^{(k)})_x + (x_i^{(k-1)})_x) \\
a_{i,y}^{(k)} &\geq (x_i^{(k+1)})_y - 2(x_i^{(k)})_y + (x_i^{(k-1)})_y \\
a_{i,y}^{(k)} &\geq -((x_i^{(k+1)})_y - 2(x_i^{(k)})_y + (x_i^{(k-1)})_y)
\end{align}
where $(x_i^{(k)})_x$ and $(x_i^{(k)})_y$ denote the $x$ and $y$ components of the state vector $x_i^{(k)}$, respectively.

This lower-level optimization stage provides trajectories that are i) dynamically feasible and respect velocity constraints; ii) region-compliant, following the assigned sequences; iii) collision-free between agents, with guarantees from our formulation; iv) safe with respect to obstacles, and v) smooth and efficient, due to the balanced objective function.

\subsection{Algorithm Overview and Implementation Details}\label{subsec:alg}

The complete multi-agent trajectory planning algorithm proceeds as follows:
\begin{algorithm}
\caption{PAAMP-Based Multi-Agent Trajectory Planning}
\begin{algorithmic}[1]
\State \textbf{Input:} $\{x_{0,i}\}$, $\{x_{f,i}\}$, regions $\mathcal R$, obstacles $\mathcal O$, max-iterations $K_{\max}$
\State \textbf{Output:} Feasible trajectories $\{x_i^{(k)}\}$
\State Construct region-adjacency graph $G_R=(V_R,E_R)$
\For{$k=1$ \textbf{to} $K_{\max}$}
    \State Generate candidate sequences $\pi_i$ for each agent
    \State Identify relevant pairs $\mathcal P^{(t)}$ via Eq.~\eqref{eq:relevant_pairs_set}
    \State Formulate MILP~\eqref{eq:multi_agent_MILP_PAAMP} with objective $J$ in Eq.~\eqref{eq:agent_cost}
    \If{MILP is \emph{feasible}}
        \State \textbf{return} trajectories $\{x_i^{(k)}\}$
    \Else
        \State Update $G_R$/ban conflicting edge (refinement step)
    \EndIf
\EndFor
\end{algorithmic}
\end{algorithm}

In the proposed implementation, we use the following practical considerations. \\
\textit{Discretization:} We use a fine enough time discretization to ensure that trajectories do not violate constraints between time steps. \\
\textit{Candidate separating directions:} For collision avoidance between agents, we use $L\!>\!2d$ candidate separating directions evenly spaced around a circle, where $d$ is the dimension of the workspace. In our 2D implementation, $L\!=\!4$ or $L\!=\!8$ directions provide a good balance between accuracy and computational efficiency. \\
\textit{Parameter tuning:} The acceleration penalty weight $\alpha$ is tuned to balance trajectory smoothness against path length. Typical values range from 0.1 to 1.0, with higher values producing smoother but potentially longer trajectories. \\
\textit{Safety margins:} We use a safety margin $\epsilon > 0$ for obstacle avoidance to ensure that agents maintain a minimum clearance from obstacles. This margin is chosen based on the physical dimensions of the agents and the precision requirements of the application.
\\
\textit{Sequence Identification:} In our approach, sequence identification is performed heuristically following a PAAMP-inspired method~\cite{Jaitly2024}, where each agent is assigned a sequence of regions to traverse based on an analysis of valid trajectory set volumes, adapted to accommodate task space planning.

\begin{remark}[Sequence–then–solve trade-off]
Fixing $\boldsymbol\pi$ greatly reduces the MILP size
but can sacrifice global optimality:  the solver seeks the
best trajectory \emph{inside} the current discrete schedule and declares
infeasibility if no collision-free motion exists, even though another
schedule might succeed. The refinement loop in Algorithm~1 mitigates this at
negligible overhead, because the MILP is only re-solved for conflicting segments.
\end{remark}

\section{Computational Complexity Analysis}\label{sec:complex}

The computational complexity of our approach is determined by two key factors: (1) the structure of the MILP formulation, and (2) the efficiency gains achieved by our PAAMP-based sequence-then-solve strategy.

\subsection{MILP Complexity Analysis}
Our MILP formulation introduces $O(N\cdot T)$ continuous decision variables
(states and auxiliaries) and, after \emph{relevant-pair pruning},
$O(|\mathcal{P}|\cdot T\cdot L)$ binary variables that encode
collision and obstacle avoidance; here $|\mathcal{P}| \ll \binom{N}{2}$ denotes the average number of relevant pairs per time step.
To quantify the efficiency gain, we define the \textit{relevant pair ratio} $\rho$ as:
\begin{equation}\label{eq:relevant_pair_ratio}
    \rho = \frac{|\mathcal{P}|}{\binom{N}{2}} = \frac{2|\mathcal{P}|}{N(N-1)},
\end{equation}
which is typically $0.5$ or smaller in our tests and cuts the exponential
branch-and-bound search space by roughly $(1-\rho)$ orders of magnitude.
Combined with region-membership constraints that tighten the LP
relaxation, the solver reaches feasibility an order of magnitude faster
than a naive big-$M$ formulation (see Sec.~\ref{sec:results}).

\subsection{Parameter Tuning and Sensitivity Analysis}
The performance and solution quality of our approach depend on several key parameters:\\
\textit{i. Big-$M$ constant and separation thresholds:} The big-$M$ constant must be sufficiently large to enforce the logical implications of the binary variables but not so large as to cause numerical issues. We typically set \(M=100\),  dimensionless in workspace units. The separation thresholds $d^{(l)}$ determine the minimum separation along each candidate direction, and setting $d^{(l)} = d_{\min}/\sqrt{2}$ for all $l$ ensures that the Euclidean separation between agents is at least $d_{\min}$. \\
\textit{ii. Acceleration penalty weight $\alpha$:} This weight balances trajectory smoothness against path length. Interestingly, we observed that for all non-zero values of $\alpha$, the optimal solution maintained the same manhattan distance cost of 16, which represents the minimum possible cost (shortest path). This indicates that the introduction of any acceleration penalty, regardless of magnitude, guides the solver toward smoother trajectories without compromising path optimality. We found that values in the range 
$[0.3, 0.7]$ provided a good trade-off between computational efficiency and trajectory smoothness. \\
\textit{iii. Number of candidate directions $L$:} Increasing $L$ provides a better approximation of the circular safety region but increases computational complexity. Our experiments showed that $L = 8$ directions for 2D problems provides sufficient accuracy without excessive computational burden. \\
\textit{iv. Region granularity:} The number and size of regions affect both solution quality and computational efficiency. Finer region decompositions can lead to better trajectories but may increase the complexity of sequence generation. We found that a balanced approach, dividing the workspace into 4-8 regions for moderate-sized environments, works well in practice.

\section{Convergence Analysis}\label{sec:converg}

In this section, we establish that the proposed PAAMP-based MILP formulation for multi-agent trajectory planning is both feasible and guaranteed to terminate in finite time under a set of reasonable assumptions. We analyze both the sequence generation phase and the continuous optimization phase, and show that our approach maintains rigorous convergence guarantees while offering computational advantages.

\subsection{Theoretical Guarantees}

\begin{theorem}
Under Assumptions 1--4 stated below, if there exists a set of admissible region sequences $\{\pi_1, \pi_2, \ldots, \pi_N\}$ such that continuous trajectories exist for all agents from their initial to goal states while satisfying dynamics, region membership, and collision avoidance constraints, then the PAAMP-based MILP formulation in~\eqref{eq:multi_agent_MILP_PAAMP} is feasible; a standard MILP solver will find a solution in finite time; and the resulting trajectories will satisfy all safety requirements, including inter-agent collision avoidance and obstacle avoidance.
\end{theorem}

\noindent\textit{Assumptions:}
\begin{enumerate}
    \item \textit{Bounded Workspace and Regions:} The workspace is bounded, and each region $R_j$ in the set $\mathcal{R}$ is a bounded polytope defined by:
    \begin{equation}
        R_j = \{x \in \mathbb{R}^{r} \;|\; A_j x \leq b_j\}
    \end{equation}
    
    \item \textit{Region Connectivity:} The Region Adjacency Graph $G_R = (V_R, E_R)$ is constructed such that an edge $(v_j, v_k) \in E_R$ exists if and only if regions $R_j$ and $R_k$ are geometrically adjacent or overlapping.
    
    \item \textit{Existence of Admissible Sequences:} There exists a set of region sequences $\{\pi_1, \pi_2, \ldots, \pi_N\}$ where each $\pi_i = \{\pi_{i,1}, \pi_{i,2}, \ldots, \pi_{i,T}\}$ satisfies: \\
    i. Initial and goal state containment: $x_{0,i} \in R_{\pi_{i,1}}$ and $x_{f,i} \in R_{\pi_{i,T}}$ \\
    ii. Sequential region adjacency: $(\pi_{i,t}, \pi_{i,t+1}) \in E_R$ for all $t = 1, \ldots, T-1$
    
    \item \textit{Appropriate Separation Parameters:} The candidate directions $\{c_{i,j}^{(l)}\}$ and separation thresholds $\{d^{(l)}\}$ in the collision avoidance constraints are chosen such that if agents are separated according to these constraints, then their actual Euclidean distance is at least $d_{\min}$.
\end{enumerate}

\begin{proof}
The proof is presented in four steps.

\paragraph*{Step 1. Region Sequence Admissibility}  
By Assumption 3, there exists a set of region sequences $\boldsymbol{\pi} {=} \{\pi_1, \pi_2, \ldots, \pi_N\}$ where each agent's sequence respects initial and goal state containment and sequential region adjacency. This ensures that each agent can traverse from its initial state to its goal state through a sequence of adjacent regions.

The admissibility of a joint sequence $\boldsymbol{\pi}$ depends on whether there exist trajectories that satisfy both the individual agent constraints and the inter-agent collision avoidance constraints. This is determined by the feasibility of the MILP in~\eqref{eq:multi_agent_MILP_PAAMP}.

\paragraph*{Step 2. Feasibility of Agent Constraints}  
For each agent $i$, and for each time step $t$, the state $x_i^{(t)}$ must satisfy:
\begin{itemize}
    \item Boundary constraints: $x_i^{(0)} = x_{0,i}$ and $x_i^{(T)} = x_{f,i}$
    \item Velocity constraints: $\|x_i^{(t+1)} - x_i^{(t)}\|_{\infty} \leq v_{\max}$
    \item Region membership: $A_j x_i^{(t)} \leq b_j$ where $j = \pi_{i,t}$
    \item Obstacle avoidance: At least one facet constraint is satisfied for each obstacle
\end{itemize}

By Assumption 3, there exists a path through the sequence of regions that connects the initial and goal states. Since each region is bounded (Assumption 1) and the velocity constraints define a bounded step size, and assuming the regions are large enough relative to $v_{\max}$ to allow transitions between adjacent regions, there exists at least one assignment of state variables $\{x_i^{(t)}\}$ that satisfies all individual agent constraints.

\paragraph*{Step 3. Feasibility of Inter-Agent Collision Avoidance}  
For each pair of agents $(i,j)$ identified as relevant at time $t$ (i.e., $(i,j) \in \mathcal{P}^{(t)}$), the collision avoidance constraints are formulated using the big-$M$ approach:
\begin{equation}
    \langle c_{i,j}^{(l)}, x_j^{(t)} - x_i^{(t)} \rangle \geq d^{(l)} - M(1-\delta_{i,j}^{(t,l)}), \quad 
    \sum_{l=1}^{L}\delta_{i,j}^{(t,l)} \geq 1
\end{equation}
Given that the continuous trajectories are collision-free (by assumption), there exists at least one index $l$ for which:
\begin{equation}
    \langle c_{i,j}^{(l)}, x_j^{(t)} - x_i^{(t)} \rangle \geq d^{(l)}
\end{equation}
For this $l$, we can set $\delta_{i,j}^{(t,l)} = 1$ to satisfy the constraint. For all other indices $l'$ where the projection does not exceed $d^{(l')}$, we can set $\delta_{i,j}^{(t,l')} = 0$, and the big-$M$ term $M(1-\delta_{i,j}^{(t,l')}) = M$ will make the constraint non-binding as long as $M$ is sufficiently large. Thus, for each relevant agent pair at each time step, there exists an assignment of variables that satisfies the collision avoidance constraints.

\paragraph*{Step 4. Finite Termination}  
The decision variables in the PAAMP-based MILP consist of: i) continuous state variables $\{x_i^{(t)}\}$ for each agent and time step; ii) auxiliary continuous variables for linearization and objective function; iii) binary variables $\{\delta_{i,j}^{(t,l)}\}$ for collision avoidance; and iv) Binary variables $\{\gamma_{i,p}^{(t,q)}\}$ for obstacle avoidance.

Since the workspace is bounded (Assumption 1) and the velocity and region membership constraints restrict the continuous decision space, the overall feasible region is bounded. Moreover, with a finite number of binary variables (significantly reduced by the relevant pair identification) the branch-and-bound procedure employed by standard MILP solvers is guaranteed to terminate in finite time.

Therefore, the PAAMP-based MILP formulation will find a feasible (or optimal) solution in finite time.
\end{proof}

\subsection{Computational Benefits of the PAAMP Approach}
The proposed approach not only maintains the theoretical guarantees of convergence but also offers significant computational advantages. Specifically: \\
\textit{Reduced Binary Variables:} By identifying and enforcing collision avoidance constraints only for relevant agent pairs, we substantially reduce the number of binary variables. If $\rho$ is the relevant pair ratio defined in~\eqref{eq:relevant_pair_ratio}, then the number of binary variables is reduced by a factor of $(1-\rho)$ compared to the naive approach. Additionally, we are not required to enforce collision avoidance with obstacles.  \\
\textit{Guided Search Space:}  
Fixing the region sequence \(\pi_i\) introduces linear
\emph{region-membership} constraints
\(A_{\pi_{i,k}}x_i^{(k)}\le b_{\pi_{i,k}}\) that are already
tight in the LP relaxation.  
Consequently, when the branch-and-bound solver explores a node it can:
(i) \emph{fathom} any branch that assigns an agent to a region not in
\(\pi_i\) without creating children, and  
(ii) compute a stronger lower bound because the relaxed problem no
longer contains the big-$M$ rows corresponding to those pruned
regions.  
Both effects shrink the search tree: variables linked to non-admissible
regions are fixed to zero at the root, and the number of active binary
variables that must be branched on is reduced from \(O(MT)\) to
\(O(|\pi_i|)\), where \(|\pi_i|=T\) by construction. \\
\textit{Accelerated Lower Bound Computation:} With fewer binary variables, the computation of lower bounds during branch-and-bound becomes faster and further speeds up the solution process.

\begin{proposition}
If the relevant pair ratio $\rho\!<\!1$, then the worst-case complexity of the PAAMP-based MILP is exponentially better than that of the naive MILP approach.
\end{proposition}

\begin{proof}
The worst-case complexity of branch-and-bound for a MILP with $b$ binary variables is $O(2^b)$. In the naive approach, $b_{\text{naive}} = O(N^2 \cdot T \cdot L)$. In the PAAMP approach, $b_{\text{PAAMP}} = O(\rho \cdot N^2 \cdot T \cdot L)$ where $\rho < 1$. Therefore:
\begin{equation}
    \frac{O(2^{b_{\text{PAAMP}}})}{O(2^{b_{\text{naive}}})} = O(2^{b_{\text{PAAMP}} - b_{\text{naive}}})
    = O(2^{(\rho - 1) \cdot N^2 \cdot T \cdot L})
\end{equation}
Since $\rho - 1 < 0$, this ratio approaches zero as $N$, $T$, or $L$ increases, which demonstrates an exponential improvement in worst-case complexity.
\end{proof}

\section{Experiments and Results}\label{sec:results}

In this study, we evaluate the proposed PAAMP-based multi-agent trajectory planning approach in a simulated two-dimensional workspace. The problem setup reflects a challenging scenario in secured and trustworthy automation (e.g., a warehouse), where multiple agents must safely navigate through a constrained environment from their respective start positions to designated goal locations while avoiding collisions with one another and with static obstacles.

\subsection{Problem Setup}
In our representative example, the simulation environment and problem parameters are defined as follows: \\
\paragraph{Workspace} The environment is a bounded 2D region defined over the interval $[0,10]$ in both the $x_1$ and $x_2$ directions.

\paragraph{Agents} We consider 4 agents, each modeled as a point mass with a separation radius of $d_{\min}/2 = 0.5$. The dynamic constraints are simplified by restricting the maximum displacement per time step to $v_{\max}=1.0$ units (in each coordinate). Each agent's trajectory is represented by its $(x_1,x_2)$ positions at discrete time instants.

\paragraph{Time Discretization} The trajectory is discretized into $T=12$ time steps to provide sufficient resolution to capture the required maneuvers while maintaining computational tractability.

\paragraph{Start and Goal Conditions} The start and goal positions for the agents are defined as follows:
\begin{itemize}
    \item Agent 0: from $(1,1)$ to $(9,9)$ (Bottom-left to Top-right)
    \item Agent 1: from $(9,1)$ to $(1,9)$ (Bottom-right to Top-left)
    \item Agent 2: from $(1,9)$ to $(9,1)$ (Top-left to Bottom-right)
    \item Agent 3: from $(9,9)$ to $(1,1)$ (Top-right to Bottom-left)
\end{itemize}
This configuration represents a challenging ``diagonal crossing'' scenario where all agents must traverse the central intersection area.

\paragraph{Collision Avoidance} Collision avoidance is enforced between agents that are in the same or adjacent regions at each time step (relevant pairs). For each relevant pair, a minimum separation distance of $d_{\min}{=}1$ is maintained using our big-$M$ formulation. This approach employs separating hyperplanes with binary variables to ensure proper agent separation.

For each agent, we use $L=8$ candidate separating directions evenly distributed around a circle, with unit vectors $c_{i,j}^{(l)} = [\cos(l\pi/4), \sin(l\pi/4)]$ for $l = 0,1,\ldots,7$. The collision avoidance constraints use binary variables $\delta_{i,j}^{(t,l)}$ to ensure that at least one separating hyperplane is active with sufficient separation.

\paragraph{Obstacle Avoidance} Four  obstacles (as $1 \times 1$ squares) are placed in the environment to create a  challenging navigation scenario.

\paragraph{Region Partitioning} As listed in Table~\ref{tab:regions}, the workspace is divided into six polytopic regions, the union of which is the free space.
\begin{table}[ht]
\centering
\caption{Workspace Polytopic Regions}
\label{tab:regions}
\setlength{\tabcolsep}{5pt}
\begin{tabular}{|c|l|l|}
\hline
\textbf{Region} & \textbf{Description} & \textbf{Constraints} \\ \hline
0 & Left vertical band    & $0 \leq x_1 \leq 2.66,\; 0 \leq x_2 \leq 10$ \\ \hline
1 & Middle vertical band  & $3.66 \leq x_1 \leq 6.33,\; 0 \leq x_2 \leq 10$ \\ \hline
2 & Right vertical band   & $7.33 \leq x_1 \leq 10,\; 0 \leq x_2 \leq 10$ \\ \hline
3 & Bottom horizontal band& $0 \leq x_2 \leq 2.66,\; 0 \leq x_1 \leq 10$ \\ \hline
4 & Middle horizontal band& $3.66 \leq x_2 \leq 6.33,\; 0 \leq x_1 \leq 10$ \\ \hline
5 & Top horizontal band   & $7.33 \leq x_2 \leq 10,\; 0 \leq x_1 \leq 10$ \\ \hline
\end{tabular}
\end{table}

Each region is defined by linear constraints in the form $Ax \leq b$, and the Region Adjacency Graph records which regions are adjacent to facilitate valid region sequences.

\paragraph{Implementation Details} The MILP is formulated and solved using the PuLP library with the CBC solver in Python. All experiments were conducted on a standard desktop computer with an Intel Core i7 processor and 16GB of RAM.

\paragraph{Impact of commercial MILP solvers} 
State-of-the-art commercial engines such as Gurobi or CPLEX incorporate stronger presolve routines, cutting-plane generation, and multi-threaded branch-and-bound.  When the proposed PAAMP formulation is handed to one of these solvers, the tighter LP relaxations and more aggressive node pruning typically translate into noticeably shorter solution times (often by a factor of several) without altering the optimal objective value or the qualitative trends reported with the open-source CBC backend. This confirms that the PAAMP formulation scales favorably and that the reported trends are solver-independent.

The weight parameter for the cost function in (\ref{eq:agent_cost}) is set to $\alpha=0.5$.
When solving the MILP, we establish an acceptable `feasibility gap' of 5 units. This prevents the MILP solver from needing to enumerate excessive binary variables to prove optimality. Instead, we accept slightly suboptimal solutions. For reference, the globally optimal cost in the following experiments was approximately 65.

\subsection{Simulation Results}
We evaluate the proposed PAAMP-based MILP approach by analyzing the trajectories generated for the four-agent crossing scenario with obstacles.

\subsubsection{Trajectory Analysis}
Figure~\ref{fig:trajectories} illustrates the planned trajectories for the four agents using our PAAMP-based approach. The start positions (squares) and goal positions (stars) are clearly marked, obstacles are shown as black rectangles, and the four region quadrants are indicated with light gray shading.

\begin{figure}[ht]
    \centering
    \includegraphics[trim={1.1cm 1.2cm 1.2cm 2.4cm},clip,width=\columnwidth]{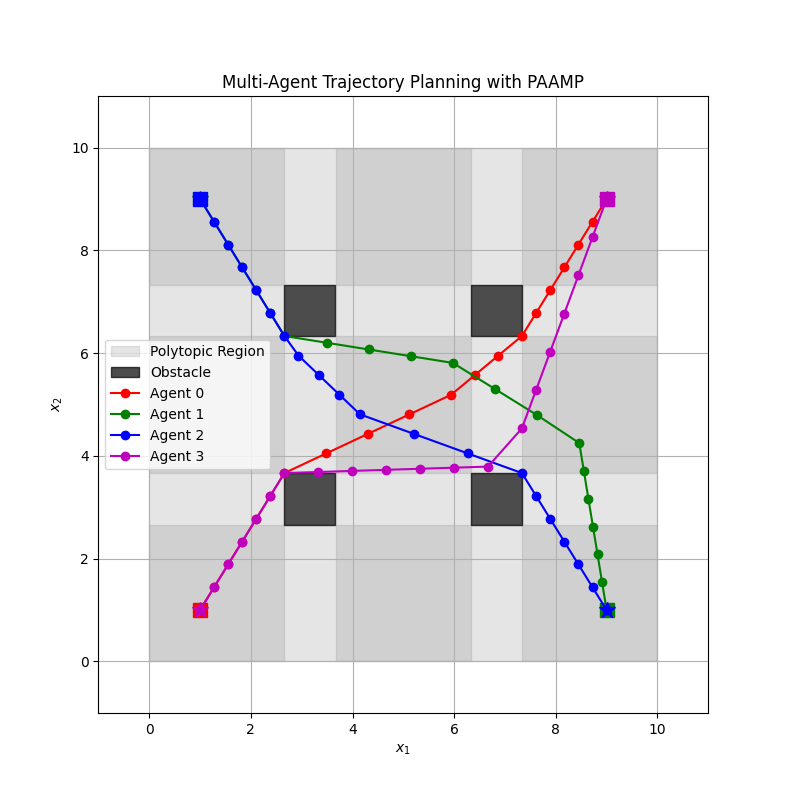}
    \caption{Planned trajectories for four agents in a 2D workspace with region partitioning. Start positions are marked with $\square$, goal positions with $\star$, obstacles in black, and regions are indicated with light gray shading.}
    \label{fig:trajectories}
\end{figure}

Several key observations can be made from the trajectory results. The agents successfully navigate around the obstacles through the central intersection area by temporally coordinating their movements, with some agents moving through earlier while others wait their turn. Each agent follows its assigned region sequence and stays within the appropriate regions at each time step as required by the PAAMP approach. The trajectories balance directness (minimizing distance) with smoothness (minimizing acceleration), which results in natural motion profiles.

Table~\ref{tab:metrics} presents the total Manhattan distance traveled by each agent and the maximum acceleration experienced during their trajectories.
For $T=20$, each agent had to consider 33 contact pairs, with an average of 1.65 pairs per time step. This is significantly lower than the maximum possible $(N-1)T = 60$ interactions per agent, which demonstrates the computational savings from the PAAMP approach.

\begin{table}[ht]
    \centering
    \caption{Trajectory Metrics for Each Agent}
    \begin{tabular}{ccc}
        \toprule
        \textbf{Agent} & \textbf{Manhattan Distance} & \textbf{Max Acceleration} \\
        \midrule
        0 & 16.0 & 0.54 \\
        1 & 16.0 & 0.74 \\
        2 & 16.0 & 0.78 \\
        3 & 16.0 & 0.72 \\
        \bottomrule
    \end{tabular}
    \label{tab:metrics}
\end{table}

Note that 16 is the lower bound for the Manhattan distance cost. An agent traveling a Manhattan distance of 16 indicates that its distance to the goal monotonically decreases throughout the trajectory.

\subsubsection{Computational Performance}
Our approach demonstrates efficient performance in solving the multi-agent trajectory planning. We compare the solving time of our proposed approach with fixed sequences of polytopes against a naive approach that uses a big-$M$ formulation to enforce obstacle avoidance. We evaluate solving times across varying trajectory fidelity levels, represented by the parameter $T$.

\begin{table}[ht]
    \centering
    \caption{Solution Times by Method and Trajectory Resolution}
    \begin{tabular}{ccc}
        \toprule
        \textbf{$T$} & \textbf{Proposed Approach (sec)} & \textbf{Naive MILP (sec)}\\
        \midrule
        12 & 0.42 & 7.06 \\
        20 & 3.60 & 54.01 \\
        30 & 7.46 & >60 \\
        50 & 21.50 & >60 \\
        \bottomrule
    \end{tabular}
    \label{tab:solve_times}
\end{table}

The PAAMP approach significantly reduces the number of binary variables by only creating collision avoidance constraints between relevant agent pairs (those in the same or adjacent regions). For our four-agent scenario, this results in 480 binary variables, which enables efficient solving despite the inherent complexity of the MILP problem.

\subsubsection{Sequence Admissibility Analysis}
An important aspect of the sequence-then-solve approach is rapidly identifying inadmissible sequences. We need programs associated with inadmissible sequences to fail quickly in order to allow efficient iteration of sequence candidates. We increase the desired minimum separation distance ($d_{\max}$) until the resulting lower-level MILP becomes definitively infeasible, and record the time required for our MILP to determine that the sequence is inadmissible.

\begin{table}[ht]
    \centering
    \caption{Time to Detect Infeasibility Using the Proposed Approach}
    \begin{tabular}{cc}
        \toprule
        \textbf{$T$} & \textbf{Infeasibility Detection Time (sec)}\\
        \midrule
        12 & 0.04 \\
        20 & 0.05 \\
        30 & 0.12 \\
        50 & 0.18 \\
        \bottomrule
    \end{tabular}
    \label{tab:time}
\end{table}

For comparison, with $T=12$, the naive MILP establishes global infeasibility in 59.4 seconds.

The simulation results demonstrate that our PAAMP-based approach with the big-$M$ collision avoidance formulation successfully generates feasible, collision-free, and smooth trajectories for multiple agents in constrained environments. The region-based sequence generation significantly reduces computational complexity while maintaining solution quality, which enables efficient planning for multi-agent systems.

\section{Conclusions and Future Work}

We presented a framework for multi-agent trajectory planning that unifies PAAMP with an efficient big‑M MILP collision‑avoidance scheme.
By reformulating the traditionally nonconvex, long-horizon planning problem into a sequence-then-solve framework, we significantly reduce the computational complexity that typically hinders multi-agent coordination. Our method divides the workspace into polytopic regions, generates agent-specific region sequences, identifies relevant agent pairs for collision avoidance, and then solves a refined Mixed-Integer Linear Program (MILP) to determine optimal trajectories or quickly assess problem feasibility. In the case of infeasibility, new sequences of regions can be generated.

The separating hyperplane approach with big-$M$ constraints provides an effective mechanism for ensuring collision avoidance between agents while maintaining the linear structure of the optimization problem. By enforcing that at least one candidate direction maintains sufficient separation between agents, we guarantee safety without introducing nonlinear constraints that would complicate the solution process.

Our convergence analysis provides theoretical guarantees that the proposed method will find a solution in finite time when one exists.
Simulation results confirm these theoretical advantages and show that our method can generate trustworthy and dynamically feasible trajectories for multiple agents navigating through constrained environments with obstacles, with computation times suitable for real-time applications.

Future research directions include developing methods to automatically determine optimal region partitioning based on environment structure and agent density, potentially using machine learning techniques to predict effective partitions. We also plan to explore more sophisticated algorithms for generating and evaluating region sequences, including approaches that explicitly consider the trade-off between sequence admissibility and the resulting relevant pair ratio.

\bibliographystyle{IEEEtran}
\bibliography{bibtex}

\begin{thebibliography}{10}
\providecommand{\url}[1]{#1}
\csname url@samestyle\endcsname
\providecommand{\newblock}{\relax}
\providecommand{\bibinfo}[2]{#2}
\providecommand{\BIBentrySTDinterwordspacing}{\spaceskip=0pt\relax}
\providecommand{\BIBentryALTinterwordstretchfactor}{4}
\providecommand{\BIBentryALTinterwordspacing}{\spaceskip=\fontdimen2\font plus
\BIBentryALTinterwordstretchfactor\fontdimen3\font minus \fontdimen4\font\relax}
\providecommand{\BIBforeignlanguage}[2]{{%
\expandafter\ifx\csname l@#1\endcsname\relax
\typeout{** WARNING: IEEEtran.bst: No hyphenation pattern has been}%
\typeout{** loaded for the language `#1'. Using the pattern for}%
\typeout{** the default language instead.}%
\else
\language=\csname l@#1\endcsname
\fi
#2}}
\providecommand{\BIBdecl}{\relax}
\BIBdecl

\bibitem{electronics12183820}
T.~Zhang, H.~Li, Y.~Fang, M.~Luo, and K.~Cao, ``Joint dispatching and cooperative trajectory planning for multiple autonomous forklifts in a warehouse: A search-and-learning-based approach,'' \emph{Electronics}, vol.~12, no.~18, 2023.

\bibitem{kagan2019autonomous}
E.~Kagan, N.~Shvalb, and I.~Ben-Gal, \emph{Autonomous Mobile Robots and Multi-Robot Systems: Motion-Planning, Communication, and Swarming}.\hskip 1em plus 0.5em minus 0.4em\relax Wiley, 2019.

\bibitem{multi2023}
L.~Antonyshyn, J.~Silveira, S.~Givigi, and J.~Marshall, ``Multiple mobile robot task and motion planning: A survey,'' \emph{ACM Comput. Surv.}, vol.~55, no.~10, Feb. 2023.

\bibitem{RRT}
S.~LaValle, ``Rapidly-exploring random trees: A new tool for path planning,'' \emph{Research Report 9811}, 1998.

\bibitem{PRM}
L.~E. Kavraki, P.~Svestka, J.-C. Latombe, and M.~H. Overmars, ``Probabilistic roadmaps for path planning in high-dimensional configuration spaces,'' \emph{IEEE transactions on Robotics and Automation}, vol.~12, no.~4, pp. 566--580, 1996.

\bibitem{direct-collocation}
M.~Kelly, ``An introduction to trajectory optimization: How to do your own direct collocation,'' \emph{SIAM Review}, vol.~59, no.~4, 2017.

\bibitem{MultiRRT2013}
M.~\v{C}\'{a}p, P.~Nov\'{a}k, J.~Vokr\'{\i}nek, and M.~P\v{e}chou\v{c}ek, ``Multi-agent {RRT}: sampling-based cooperative pathfinding,'' in \emph{Proceedings of the 2013 International Conference on Autonomous Agents and Multi-Agent Systems}, ser. AAMAS '13, 2013, p. 1263–1264.

\bibitem{Yang2008}
P.~Yang, R.~A. Freeman, and K.~M. Lynch, ``Multi-agent coordination by decentralized estimation and control,'' \emph{IEEE Transactions on Automatic Control}, vol.~53, no.~11, pp. 2480--2496, 2008.

\bibitem{Augugliaro2012}
F.~Augugliaro, A.~P. Schoellig, and R.~D'Andrea, ``Generation of collision-free trajectories for a quadrocopter fleet: A sequential convex programming approach,'' in \emph{2012 IEEE/RSJ International Conference on Intelligent Robots and Systems}, 2012, pp. 1917--1922.

\bibitem{Turnbull2013}
O.~Turnbull and A.~Richards, ``Collocation methods for multi-vehicle trajectory optimization,'' in \emph{2013 European Control Conference (ECC)}, 2013, pp. 1230--1235.

\bibitem{Robinson2018}
D.~R. Robinson, R.~T. Mar, K.~Estabridis, and G.~Hewer, ``An efficient algorithm for optimal trajectory generation for heterogeneous multi-agent systems in non-convex environments,'' \emph{IEEE Robotics and Automation Letters}, vol.~3, no.~2, pp. 1215--1222, 2018.

\bibitem{MILP_book}
G.~L. Nemhauser and L.~A. Wolsey, \emph{Integer and combinatorial optimization}.\hskip 1em plus 0.5em minus 0.4em\relax Wiley, 1999.

\bibitem{GJK}
E.~Gilbert, D.~Johnson, and S.~Keerthi, ``A fast procedure for computing the distance between complex objects in three-dimensional space,'' \emph{IEEE Journal on Robotics and Automation}, vol.~4, no.~2, pp. 193--203, 1988.

\bibitem{Tracy2023}
K.~Tracy, T.~A. Howell, and Z.~Manchester, ``Differentiable collision detection for a set of convex primitives,'' in \emph{2023 IEEE International Conference on Robotics and Automation (ICRA)}, 2023, pp. 3663--3670.

\bibitem{LeCleac2023}
S.~Le~Cleac'h, M.~Schwager, Z.~Manchester, V.~Sindhwani, P.~Florence, and S.~Singh, ``Single-level differentiable contact simulation,'' \emph{IEEE Robotics and Automation Letters}, vol.~8, no.~7, pp. 4012--4019, 2023.

\bibitem{SIPP}
M.~Phillips and M.~Likhachev, ``Sipp: Safe interval path planning for dynamic environments,'' in \emph{2011 IEEE International Conference on Robotics and Automation}, 2011, pp. 5628--5635.

\bibitem{sharon2015conflict}
G.~Sharon, R.~Stern, A.~Felner, and N.~R. Sturtevant, ``Conflict-based search for optimal multi-agent pathfinding,'' \emph{Artificial intelligence}, vol. 219, pp. 40--66, 2015.

\bibitem{Jaitly2024}
A.~Jaitly and S.~Farzan, ``{PAAMP}: Polytopic action-set and motion planning for long horizon dynamic motion planning via mixed integer linear programming,'' in \emph{2024 IEEE/RSJ International Conference on Intelligent Robots and Systems (IROS)}, 2024, pp. 7617--7624.

\bibitem{TAMPsurvey}
C.~R. Garrett, R.~Chitnis, R.~Holladay, B.~Kim, T.~Silver, L.~P. Kaelbling, and T.~Lozano-P\'{e}rez, ``Integrated task and motion planning,'' \emph{Annual Review of Control, Robotics, and Autonomous Systems}, vol.~4, no.~1, pp. 265--293, 2021.

\bibitem{GCS}
T.~Marcucci, J.~Umenberger, P.~Parrilo, and R.~Tedrake, ``Shortest paths in graphs of convex sets,'' \emph{SIAM Journal on Optimization}, vol.~34, no.~1, pp. 507--532, 2024.

\bibitem{GCS_state_space}
T.~Marcucci, M.~Petersen, D.~von Wrangel, and R.~Tedrake, ``Motion planning around obstacles with convex optimization,'' \emph{Science Robotics}, vol.~8, no.~84, p. eadf7843, 2023.

\end{thebibliography}

\appendix
\section*{Supplementary Figures}

This appendix gathers figures that support the formulation. Notation follows the main text: regions \(R_j\), region adjacency graph \(G_R\), relevant pair set \(\mathcal P^{(t)}\), separation radius \(d_{\min}\), candidate directions \(c^{(l)}\), obstacle facet normals \(a_q\), velocity bound \(v_{\max}\), and smoothing weight \(\alpha\). No new assumptions are introduced.

Figures are ordered from constructs to complexity:
\begin{itemize}[leftmargin=0pt, labelsep=0.5em]
  \item[] \textbf{Fig.~\ref{fig:workspace_adj}}: Sec.~\ref{subsec:region-rep} (PAAMP region representation). Shows the workspace partition and its \(G_R\); edges indicate \(R_j \cap R_k \neq \emptyset\). Used to define admissible transitions for \(\pi_i\) and to test adjacency in \(\mathcal P^{(t)}\).
  \item[] \textbf{Fig.~\ref{fig:block}}: A one-column pipeline summary aligned with Sec.~\ref{sec:milp}.
  \item[] \textbf{Fig.~\ref{fig:region_segment}}: Sec.~\ref{subsec:region-rep} (region membership). Demonstrates that constraining \(x_i^{(k)}, x_i^{(k+1)} \in R_j\) keeps the segment within \(R_j\). Invokes convexity used in the region-membership constraints.
  \item[] \textbf{Fig.~\ref{fig:relevant_pairs_time}}: Sec.~\ref{subsec:sequence} (Relevant Pair Identification). Visualizes \(\mathcal P^{(t)}\) across time and why \(|\mathcal P^{(t)}| < \binom{N}{2}\) in practice; motivates the relevant-pair ratio \(\rho\).
  \item[] \textbf{Fig.~\ref{fig:sep_hyperplane}}: Sec.~\ref{subsec:collision} (Collision avoidance). Depicts the separating-hyperplane condition \(\langle c^{(l)}, x_j^{(k)} - x_i^{(k)} \rangle \ge d^{(l)}\) with inactive directions ignored; corresponds to the big-\(M\) rows in the MILP.\looseness=-1
  \item[] \textbf{Fig.~\ref{fig:dir_sampling}}: Sec.~\ref{subsec:collision} (Direction sampling). Shows how sampled half-spaces outer-approximate the Euclidean safety disk; increasing \(L\) tightens the polygon and informs the choice of \(d^{(l)}\).\looseness=-1
  \item[] \textbf{Fig.~\ref{fig:obstacle_facets}}: Sec.~\ref{subsec:sequence} (Obstacle avoidance within MILP). Illustrates enforcing at least one outward obstacle facet with margin \(\epsilon\); the shaded half-space matches the active inequality.
  \item[] \textbf{Fig.~\ref{fig:cost_smoothing}}: Sec.~\ref{subsec:sequence} (Lower-level objective). Highlights the \(L_1\) step term and discrete acceleration term, and the effect of \(\alpha\) on path smoothness; corresponds to \eqref{eq:agent_cost}.
  \item[] \textbf{Fig.~\ref{fig:refine_on_conflict}}: Sec.~\ref{subsec:alg} (Refinement loop). Shows blacklisting of the offending time-indexed transition at \(t^\star\), replanning sequences, and resolving; matches Algorithm~1.
  \item[] \textbf{Fig.~\ref{fig:bb_compare}}: Sec.~\ref{sec:complex} (Complexity). Contrasts branch-and-bound trees for the naive MILP versus the PAAMP-based formulation with relevant-pair pruning; relates to the exponential reduction in binaries.
  \item[] \textbf{Fig.~\ref{fig:overall_method}}: End-to-end methodology that mirrors Sec.~\ref{sec:milp}: partition and \(G_R\), sequence generation, relevant-pair pruning, MILP assembly, solver, and the refine step when infeasible.
\end{itemize}

\begin{figure}[b]
\centering
\begin{tikzpicture}[scale=0.45, font=\footnotesize]
\tikzset{
  region/.style={draw, thick, fill=gray!12},
  gridline/.style={draw, gray!50, thin},
  labelnode/.style={font=\footnotesize},
  grnode/.style={circle, draw, thick, fill=white, inner sep=1.5pt, minimum size=6pt},
  gredge/.style={line width=0.8pt}
}

\draw[thick] (0,0) rectangle (8,8);
\node[anchor=south west] at (-1.8,8.2) {Workspace partition into convex polytopes};

\draw[gridline] (2.66,0) -- (2.66,8);
\draw[gridline] (5.33,0) -- (5.33,8);
\draw[gridline] (0,4) -- (8,4);

\filldraw[region] (0,4) rectangle (2.66,8);
\filldraw[region] (2.66,4) rectangle (5.33,8);
\filldraw[region] (5.33,4) rectangle (8,8);
\filldraw[region] (0,0) rectangle (2.66,4);
\filldraw[region] (2.66,0) rectangle (5.33,4);
\filldraw[region] (5.33,0) rectangle (8,4);

\node[labelnode] at (1.33,6) {$R_1$};
\node[labelnode] at (4.0,6)   {$R_2$};
\node[labelnode] at (6.65,6)  {$R_3$};
\node[labelnode] at (1.33,2) {$R_4$};
\node[labelnode] at (4.0,2)   {$R_5$};
\node[labelnode] at (6.65,2)  {$R_6$};

\begin{scope}[xshift=10.5cm, yshift=0cm]
  \node[anchor=south west] at (-0.75,7) {$G_R$ (region adjacency)};
  \node[grnode] (r1) at (0.0,6) {$R_1$};
  \node[grnode] (r2) at (2.5,6) {$R_2$};
  \node[grnode] (r3) at (5,6) {$R_3$};
  \node[grnode] (r4) at (0.0,2) {$R_4$};
  \node[grnode] (r5) at (2.5,2) {$R_5$};
  \node[grnode] (r6) at (5,2) {$R_6$};

  \draw[gredge] (r1) -- (r2);
  \draw[gredge] (r2) -- (r3);
  \draw[gredge] (r1) -- (r4);
  \draw[gredge] (r2) -- (r5);
  \draw[gredge] (r3) -- (r6);
  \draw[gredge] (r4) -- (r5);
  \draw[gredge] (r5) -- (r6);
\end{scope}
\end{tikzpicture}
\caption{Workspace partition into convex polytopes \(R_j\) and the corresponding region adjacency graph \(G_R\). Edges indicate \(R_j\cap R_k\neq\emptyset\) (here, shared boundaries in the grid).}
\label{fig:workspace_adj}
\end{figure}
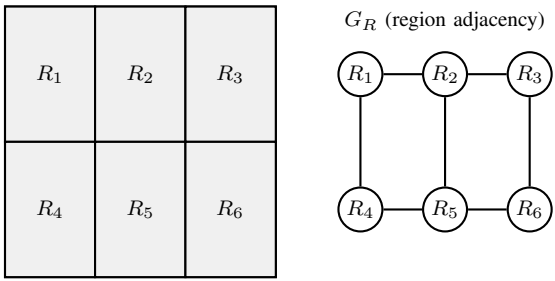

\begin{figure}
\centering
\begin{tikzpicture}[
  font=\scriptsize,
  node distance=4mm,
  >=Stealth,
  box/.style={draw, rounded corners=2pt, fill=gray!10, minimum width=42mm, minimum height=6mm, align=center},
  dbox/.style={diamond, draw, aspect=2, inner sep=1.2pt, fill=gray!10, align=center},
  line/.style={-Latex, line width=0.55pt}
]

\node[box] (in)    {Map, starts/goals,\; $T,L,\epsilon,\alpha,v_{\max}$};
\node[box, below=of in]   (part)  {Convex partition + $G_R$};
\node[box, below=of part] (seq)   {Sequences $\pi_i$ (shortest paths)};
\node[box, below=of seq]  (pairs) {Relevant pairs $\mathcal P^{(t)}$};
\node[box, below=of pairs](asm)   {Assemble MILP $\;(x,\delta,\gamma)$};
\node[box, below=of asm]  (sol)   {MILP solver};
\node[dbox, below=of sol,xshift=0mm] (ok) {Feasible?};
\node[box, below=of sol, yshift=-12mm]    (traj) {Trajectories (yes)};

\node[box, minimum width=12mm, right=5mm of ok] (ref) {Refine (ban edge)};
\draw[line] (ref.north) .. controls +(0,13mm) and +(10mm,-10mm) .. (seq.east);

\draw[line] (in)   -- (part)
            (part) -- (seq)
            (seq)  -- (pairs)
            (pairs)-- (asm)
            (asm)  -- (sol)
            (sol)  -- (ok)
            (ok)   -- node[right, xshift=0.4mm]{\scriptsize yes} (traj);

\draw[line] (ok.east) -- node[above]{\scriptsize no} (ref.west);
\end{tikzpicture}
\caption{Sequence then solve pipeline. On infeasibility, refine by blacklisting the offending time-region transition and replan.}
\label{fig:block}
\end{figure}
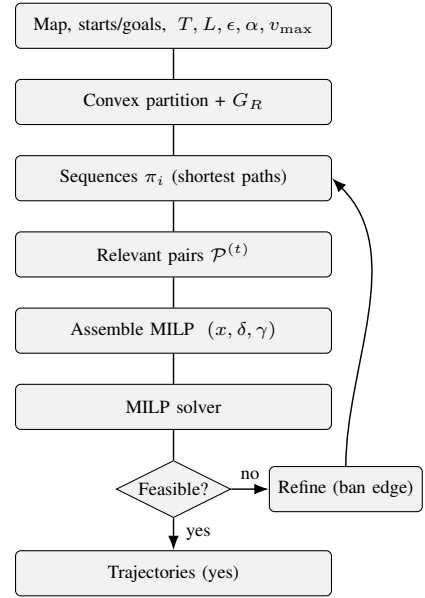

\begin{figure}[t]
\centering
\begin{tikzpicture}[font=\footnotesize]

\tikzset{
  region/.style={draw, very thick, fill=gray!12},
  okseg/.style={very thick},
  badseg/.style={very thick, dash pattern=on 3pt off 2pt},
  pt/.style={circle, fill=black, inner sep=1.4pt},
  note/.style={font=\scriptsize}
}

\coordinate (p1) at (0.3,0.4);
\coordinate (p2) at (4.6,0.4);
\coordinate (p3) at (4.9,2.8);
\coordinate (p4) at (3.0,4.7);
\coordinate (p5) at (0.4,3.6);

\path[region] (p1) -- (p2) -- (p3) -- (p4) -- (p5) -- cycle;
\node at (2.6,3.9) {$R_j$};
\node[note, anchor=south west] at (3.6,4.1) {$R_j=\{x \mid A_j x \le b_j\}$};

\coordinate (xk)  at (1.1,1.1);
\coordinate (xk1) at (3.2,2.0);

\draw[okseg] (xk) -- (xk1);
\node[pt, label=left:{$x_i^{(k)}$}] at (xk) {};
\node[pt, label=above:{$x_i^{(k+1)}$}] at (xk1) {};

\node[note, align=center] at (2.9,-0.0) {Endpoints in $R_j$ $\Rightarrow$ all convex combinations stay in $R_j$};

\coordinate (xk2bad) at (5.6,3.3);
\draw[badseg] (xk1) -- (xk2bad);
\node[pt] at (xk2bad) {};
\node[note, anchor=west] at (5.2,2.9) {$x_i^{(k+2)} \notin R_j$};
\node[note, anchor=south east] at (4.35,2.85) {violates region membership};
\draw[->, thin] (4.35,3) -- (4.65,2.9);
\draw[black] (intersection of xk1--xk2bad and p3--p4) circle (1.5pt);
\node[note, anchor=west] at (4.2,3.6) {dashed path leaves $R_j$ (invalid)};
\end{tikzpicture}
\caption{Region-compliant segment: enforcing $x_i^{(k)},x_i^{(k+1)} \in R_j$ keeps the segment inside the convex polytope $R_j$. A continuation that steps outside violates region membership.}
\label{fig:region_segment}
\end{figure}

\begin{figure}[t]
\centering
\begin{tikzpicture}[x=0.9cm,y=0.9cm, font=\footnotesize, >=Stealth]

\tikzset{
  gridline/.style={draw, gray!60, thin},
  head/.style={font=\scriptsize},
  relpair/.style={line width=1pt,color=blue,dashed},
  note/.style={font=\scriptsize}
}

\pgfmathtruncatemacro{\T}{8}
\pgfmathtruncatemacro{\Tminusone}{\T-1}
\def\N{4}

\foreach \t in {1,...,\T} {
  \node[head] at (\t-0.5,4.25) {$t=\t$};
}
\node[head, anchor=east] at (-0.2,3.5) {Agent 1};
\node[head, anchor=east] at (-0.2,2.5) {Agent 2};
\node[head, anchor=east] at (-0.2,1.5) {Agent 3};
\node[head, anchor=east] at (-0.2,0.5) {Agent 4};

\fill[gray!15] (6,0) rectangle (7,4);

\draw[gridline] (0,0) rectangle (\T,4);
\foreach \x in {1,...,\Tminusone} { \draw[gridline] (\x,0) -- (\x,4); }
\foreach \y in {1,...,3}   { \draw[gridline] (0,\y) -- (\T,\y); }

\foreach[count=\c from 1] \r in {1,1,2,2,2,3,3,3} {
  \node at (\c-0.5,3.5) {$R_{\r}$};
}
\foreach[count=\c from 1] \r in {3,2,2,2,1,1,1,1} {
  \node at (\c-0.5,2.5) {$R_{\r}$};
}
\foreach[count=\c from 1] \r in {2,2,1,1,2,2,2,3} {
  \node at (\c-0.5,1.5) {$R_{\r}$};
}
\foreach[count=\c from 1] \r in {1,1,1,2,2,2,3,3} {
  \node at (\c-0.5,0.5) {$R_{\r}$};
}

\def\xc{6.5}
\draw[relpair] (\xc,3.5) .. controls (7.2,3.5) and (7.2,1.5) .. (\xc,1.5);
\draw[relpair] (\xc,3.5) .. controls (7.5,3.5) and (7.5,0.5) .. (\xc,0.5);
\draw[relpair] (\xc,2.5) .. controls (6.9,2.5) and (6.9,1.5) .. (\xc,1.5);
\draw[relpair] (\xc,1.5) .. controls (6.9,1.5) and (6.9,0.5) .. (\xc,0.5);

\node[note, anchor=west,blue] at (-1,-1.35) {$\mathcal{P}^{(7)}=\{(1,3),(1,4),(2,3),(3,4)\}$};

\node[head, anchor=east] at (-0.2,-0.7) {$|\mathcal{P}^{(t)}|$};
\foreach[count=\c from 1] \val in {4,6,6,6,6,5,4,3} {
  \node at (\c-0.5,-0.7) {\val};
}
\node[note, anchor=west] at (\T-2.75,-1.35) {Total per $t$: $\binom{4}{2}=6$};

\end{tikzpicture}
\caption{Relevant-pair pruning across time. Only agent pairs sharing a region or an adjacent region at time $t$ enter $\mathcal{P}^{(t)}$. The highlighted column illustrates active pairs at $t=7$. The bottom row shows $|\mathcal{P}^{(t)}|$, which can be strictly less than the total $\binom{N}{2}$.}
\label{fig:relevant_pairs_time}
\end{figure}

\begin{figure}[t]
\centering
\begin{tikzpicture}[scale=0.7, font=\footnotesize, >=Stealth]

\tikzset{
  pt/.style={circle, fill=black, inner sep=1.3pt},
  hyper/.style={very thick},
  hyperInactive/.style={very thick, dashed},
  safe/.style={fill=gray!15, draw=none},
  note/.style={font=\scriptsize}
}

\coordinate (xi) at (0,0);
\coordinate (xj) at (3,1.2);

\def\cx{0.8}
\def\cy{0.6}
\def\px{-0.6}
\def\py{0.8}
\def\dd{2.0}
\def\S{3.0}
\def\L{6.0}

\coordinate (H)  at ({\dd*\cx},{\dd*\cy});
\coordinate (H1) at ({\dd*\cx + \S*\px},{\dd*\cy + \S*\py});
\coordinate (H2) at ({\dd*\cx - \S*\px},{\dd*\cy - \S*\py});
\coordinate (H3) at ({\dd*\cx - \S*\px + \L*\cx},{\dd*\cy - \S*\py + \L*\cy});
\coordinate (H4) at ({\dd*\cx + \S*\px + \L*\cx},{\dd*\cy + \S*\py + \L*\cy});

\fill[safe] (H1) -- (H2) -- (H3) -- (H4) -- cycle;
\draw[hyper] (H1) -- (H2);

\draw[->] (xi) -- ++({1.1*\cx},{1.1*\cy}) node[above left] {$c^{(l)}$};
\draw[->] (xi) -- (H) node[midway, below right] {$d^{(l)}$};

\node[pt, label=below:{$x_i^{(k)}$}] at (xi) {};
\node[pt, label=below:{$x_j^{(k)}$}] at (xj) {};

\def\cxb{-0.6}
\def\cyb{0.8}
\def\pxb{-0.8}
\def\pyb{-0.6}
\coordinate (Hb)  at ({\dd*\cxb},{\dd*\cyb});
\coordinate (Hb1) at ({\dd*\cxb + \S*\pxb},{\dd*\cyb + \S*\pyb});
\coordinate (Hb2) at ({\dd*\cxb - \S*\pxb},{\dd*\cyb - \S*\pyb});
\draw[->, gray!70] (xi) -- ++({0.9*\cxb},{0.9*\cyb}) node[above] {$c^{(m)}$};
\node[note, anchor=south east] at (1.2,3.4) {inactive};

\draw[->, thin] ($(H)!0.65!(H3)$) -- ++({0.7*\cx},{0.7*\cy});
\draw[->, thin] ($(H)!0.35!(H4)$) -- ++({0.7*\cx},{0.7*\cy});
\node[note] at (4.4,6.5) {safe half-space for $x_j$};
\node[note, anchor=west] at ($(H2)!0.5!(H3)+(-4.2,4.2)$)
{$\{x:\langle c^{(l)},x{-}x_i^{(k)}\rangle \ge d^{(l)}\}$};
\draw[hyperInactive, color=gray!90] (Hb1) -- (Hb2)
  node[midway, above left, note, xshift=10pt,yshift=10pt]{inactive $c^{(m)}$ (ignored)};
\end{tikzpicture}
\caption{Separating hyperplane for inter-agent avoidance. For an active candidate direction $c^{(l)}$, the safe half-space satisfies $\langle c^{(l)},\, x_j^{(k)}-x_i^{(k)}\rangle \ge d^{(l)}$. A second candidate direction $c^{(m)}$ is shown inactive.}
\label{fig:sep_hyperplane}
\end{figure}
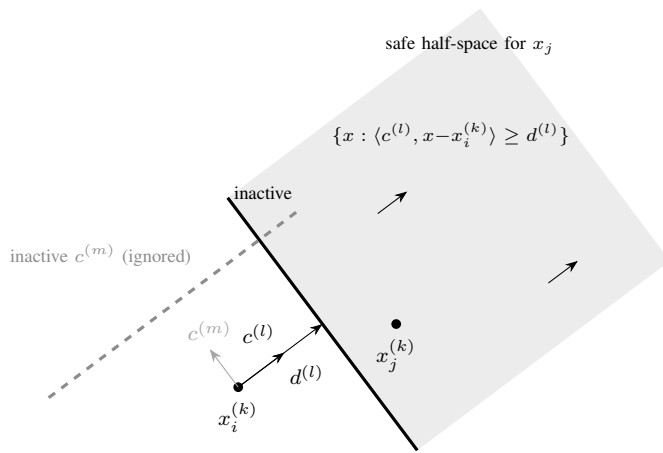

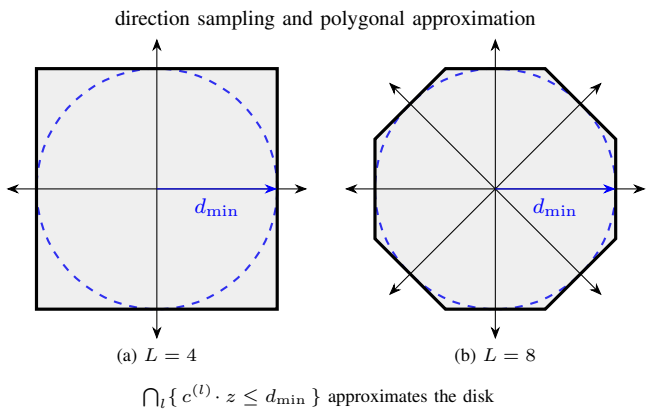
\begin{figure}[t]
\centering
\begin{tikzpicture}[font=\footnotesize, >=Stealth]

\tikzset{
  disk/.style={draw, thick, dashed},
  poly/.style={draw, very thick, fill=gray!50, fill opacity=0.25},
  ray/.style={->, thin},
  note/.style={font=\scriptsize}
}

\begin{scope}
  \pgfmathtruncatemacro{\Lsamp}{4}
  \pgfmathsetmacro{\dmin}{1.6}
  \pgfmathsetmacro{\R}{\dmin / cos(180/\Lsamp)}

  \draw[disk,blue] (0,0) circle[radius=\dmin];

  \path[poly]
    ({\R*cos((2*0+1)*180/\Lsamp)},{\R*sin((2*0+1)*180/\Lsamp)})
    \foreach \k in {1,...,\numexpr\Lsamp-1\relax} {
      -- ({\R*cos((2*\k+1)*180/\Lsamp)},{\R*sin((2*\k+1)*180/\Lsamp)})
    } -- cycle;

  \foreach \k in {0,...,\numexpr\Lsamp-1\relax} {
    \draw[ray] (0,0) -- ({1.25*\dmin*cos(360*\k/\Lsamp)},{1.25*\dmin*sin(360*\k/\Lsamp)});
  }

  \draw[->,blue] (0,0) -- (\dmin,0) node[midway, below] {$d_{\min}$};
  \node[note, anchor=north] at (0,-2.0) {(a) $L=4$};
\end{scope}

\begin{scope}[xshift=4.5cm]
  \pgfmathtruncatemacro{\Lsamp}{8}
  \pgfmathsetmacro{\dmin}{1.6}
  \pgfmathsetmacro{\R}{\dmin / cos(180/\Lsamp)}

  \draw[disk, blue] (0,0) circle[radius=\dmin];

  \path[poly]
    ({\R*cos((2*0+1)*180/\Lsamp)},{\R*sin((2*0+1)*180/\Lsamp)})
    \foreach \k in {1,...,\numexpr\Lsamp-1\relax} {
      -- ({\R*cos((2*\k+1)*180/\Lsamp)},{\R*sin((2*\k+1)*180/\Lsamp)})
    } -- cycle;

  \foreach \k in {0,...,\numexpr\Lsamp-1\relax} {
    \draw[ray] (0,0) -- ({1.25*\dmin*cos(360*\k/\Lsamp)},{1.25*\dmin*sin(360*\k/\Lsamp)});
  }

  \draw[->, blue] (0,0) -- (\dmin,0) node[midway, below] {$d_{\min}$};
  \node[note, anchor=north] at (0,-2.0) {(b) $L=8$};
\end{scope}

\node[align=left, anchor=west] at (-0.68,2.25) {direction sampling and polygonal approximation};
\node[note, align=center, anchor=west] at (-0.32,-2.75)
{$\bigcap_l\{\,c^{(l)}\!\cdot z \le d_{\min}\,\}$ approximates the disk};

\end{tikzpicture}
\caption{Direction sampling for separation. The dashed circle is radius \(d_{\min}\); the filled polygon is \(\bigcap_l\{\,c^{(l)}\!\cdot z \le d_{\min}\,\}\), an outer approximation of the disk for \(z=x_j^{(k)}-x_i^{(k)}\). Larger \(L\) tightens the approximation.}
\label{fig:dir_sampling}
\end{figure}

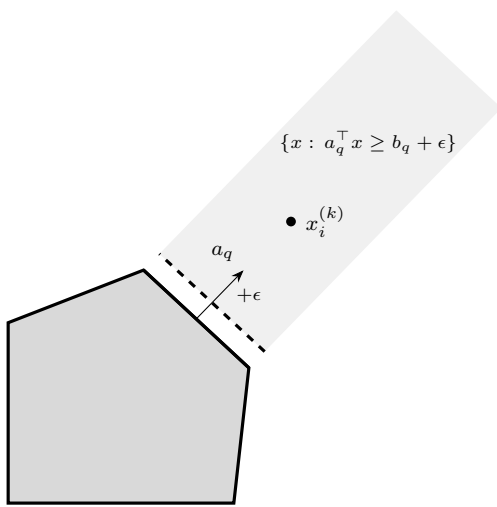
\begin{figure}[t]
\centering
\begin{tikzpicture}[font=\footnotesize, >=Stealth]

\tikzset{
  obstacle/.style={draw, very thick, fill=gray!30},
  epsline/.style={very thick, dashed},
  activefacet/.style={very thick},
  normal/.style={->, thin},
  safe/.style={fill=gray!12, draw=none},
  pt/.style={circle, fill=black, inner sep=1.3pt},
  note/.style={font=\scriptsize}
}

\coordinate (A) at (0,0);
\coordinate (B) at (3.0,0);
\coordinate (C) at (3.2,1.8);
\coordinate (D) at (1.8,3.1);
\coordinate (E) at (0,2.4);

\path[name path=obs] (A) -- (B) -- (C) -- (D) -- (E) -- cycle;

\def\nx{0.70}
\def\ny{0.72}
\def\eps{0.30}
\def\L{4.5}

\coordinate (Ceps) at ($ (C) + \eps*(\nx,\ny) $);
\coordinate (Deps) at ($ (D) + \eps*(\nx,\ny) $);
\coordinate (Cfar) at ($ (Ceps) + \L*(\nx,\ny) $);
\coordinate (Dfar) at ($ (Deps) + \L*(\nx,\ny) $);

\path[safe] (Ceps) -- (Deps) -- (Dfar) -- (Cfar) -- cycle;

\filldraw[obstacle] (A) -- (B) -- (C) -- (D) -- (E) -- cycle;
\draw[activefacet] (C) -- (D);
\draw[epsline] (Ceps) -- (Deps);
\node[note, anchor=west] at ($ (Deps)!0.5!(Ceps) + (0.2,0.1) $) {$+\epsilon$};

\draw[normal] ($ (C)!0.5!(D) $) -- ++(0.9*\nx,0.9*\ny) node[above left] {$a_q$};
\coordinate (xk) at ($ (Ceps)!0.5!(Deps) + (1.5*\nx, 1.5*\ny) $);
\node[pt, label=right:{$x_i^{(k)}$}] at (xk) {};

\node[note, anchor=west] at ($(Ceps)!0.5!(Dfar)+(-0.8,0.5)$)
  {$\{x:\,a_q^\top x \ge b_q+\epsilon\}$};

\end{tikzpicture}
\caption{Obstacle avoidance via facet selection. To stay outside a convex obstacle, at least one outward facet inequality is enforced with a margin $\epsilon$. The active facet’s $\epsilon$-offset half-space (shaded) contains the agent state; other facets are inactive.
\(a_q\) denotes the outward normal of facet \(q\).}
\label{fig:obstacle_facets}
\end{figure}

\begin{figure}[t]
\centering
\begin{tikzpicture}[font=\footnotesize, >=Stealth]

\begin{scope}[yshift=3.8cm]
  \tikzset{
    traj/.style={very thick},
    stepbrace/.style={decorate, decoration={brace, amplitude=3pt}},
    pt/.style={circle, fill=black, inner sep=1.2pt},
    note/.style={font=\scriptsize}
  }

  \draw[->] (-0.2,0) -- (5.6,0) node[right] {$k$};
  \coordinate (k0) at (0,0.6);
  \coordinate (k1) at (1,1.6);
  \coordinate (k2) at (2,1.2);
  \coordinate (k3) at (3,2.4);
  \coordinate (k4) at (4,2.1);
  \coordinate (k5) at (5,3.1);

  \draw[traj] (k0) -- (k1) -- (k2) -- (k3) -- (k4) -- (k5);
  \foreach \p/\lab in {k0/0,k1/1,k2/2,k3/3,k4/4,k5/5}{
    \node[pt] at (\p) {};
    \node[below] at ($(\p)+(0,-0.10)$) {$\lab$};
  }

  \def\xannot{4.9}

  \draw[dashed,blue] (k2) -- (\xannot,1.2);
  \draw[dashed,blue] (k1) -- (\xannot,1.6);
  \draw[stepbrace] (\xannot,1.2) -- node[right=2pt] {$\big|x^{(2)}{-}x^{(1)}\big|$} (\xannot,1.6);

  \draw[dashed,blue] (k4) -- (\xannot,2.1);
  \draw[dashed,blue] (k3) -- (\xannot,2.4);
  \draw[stepbrace] (\xannot,2.1) -- node[right=2pt] {$\big|x^{(4)}{-}x^{(3)}\big|$} (\xannot,2.4);

  \node[note, fill=white, inner sep=2pt] (acc) at (\xannot,1.85)
    {$\big|x^{(3)}{-}2x^{(2)}{+}x^{(1)}\big|$};
  \draw[dashed,blue] (k1) -- (acc.west);
  \draw[dashed,blue] (k2) -- (acc.west);
  \draw[dashed,blue] (k3) -- (acc.west);

  \node[note, anchor=west] at (-0.4,3.25) {(a) Discrete cost components};
\end{scope}

\begin{scope}[yshift=0.0cm]
  \tikzset{
    shortpath/.style={thick, dash pattern=on 4pt off 4pt, blue},
    smoothpath/.style={very thick},
    start/.style={star,star points=5,star point ratio=2.4,fill=black,minimum size=5pt,inner sep=0pt},
    goal/.style={star,star points=5,star point ratio=2.4,fill=white,draw=black,minimum size=6pt,inner sep=0pt},
    note/.style={font=\scriptsize}
  }

  \node[start, label=left:{$x_0$}] (S) at (0.2,0.2) {};
  \node[goal,  label=right:{$x_f$}] (G) at (6.0,3.2) {};

  \draw[shortpath]
    (S) -- (2.0,0.2) -- (2.0,1.8) -- (4.2,1.8) -- (4.2,3.0) -- (G);
  \node[note, anchor=west] at (0.2,2.75) {$\alpha=0$ (length only)};

  \draw[smoothpath]
    (S) .. controls (1.2,0.7) and (2.5,0.9) .. (2.8,1.6)
        .. controls (3.1,2.2) and (4.8,2.6) .. (G);
  \node[note, anchor=east] at (6.0,0.2) {$\alpha>0$ (length + smoothing)};

  \draw[->, thin] (3.1,1.8) -- +(0,0.75) node[above, note] {jerkier};
  \draw[->, thin] (3.8,2.2) -- +(0,-0.75) node[below, note] {smoother};

  \node[note, anchor=west] at (-0.5,3.25) {(b) Effect of $\alpha$ on 2D paths};
\end{scope}

\end{tikzpicture}
\caption{Cost components and smoothing. (a) The $L_1$ step term penalizes per-step displacement; the acceleration term penalizes discrete curvature via second differences. (b) Increasing $\alpha$ yields smoother trajectories that may be slightly longer but have reduced acceleration.}
\label{fig:cost_smoothing}
\end{figure}
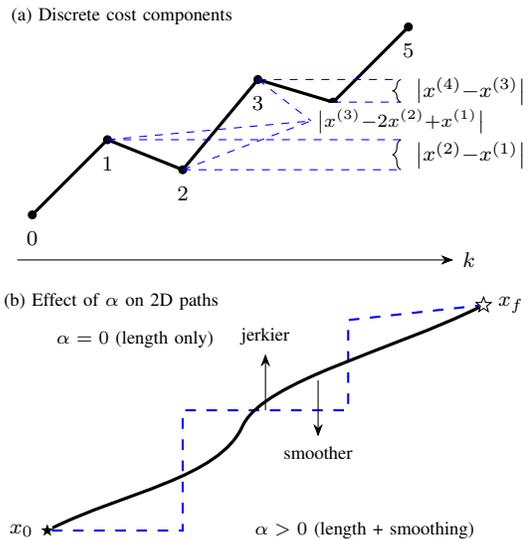

\begin{figure}[t]
\centering
\begin{tikzpicture}[font=\footnotesize, >=Stealth]
\tikzset{
  regionnode/.style={circle, draw, thick, fill=white, inner sep=1.5pt, minimum size=7.5mm},
  edge/.style={line width=0.8pt},
  banned/.style={line width=1.0pt, dash pattern=on 3pt off 2pt},
  seqbox/.style={draw, rounded corners=2pt, fill=gray!12, inner sep=2.5pt},
  seqarrow/.style={->, thick},
  conflictmark/.style={draw, cross out, line width=0.9pt, minimum size=7pt, inner sep=0pt},
  note/.style={font=\scriptsize}
}

\node[seqbox] (seqA) at (0,6.5) {
  \begin{tikzpicture}[x=1cm,y=1cm, baseline=(current bounding box.center)]
    \foreach \t in {1,...,7} {\node at (\t,0.3) {$t=\t$};}
    \foreach \t/\r in {1/$R_1$,2/$R_2$,3/$R_3$,4/$R_3$,5/$R_3$,6/$R_2$,7/$R_2$} {
      \node[draw, rounded corners=1.5pt, fill=white, minimum width=7mm, minimum height=4mm] (b\t) at (\t,1) {\r};
    }
    \foreach \t in {1,...,6} {
      \pgfmathtruncatemacro{\n}{\t+1}
      \draw[seqarrow] (b\t.east) -- (b\n.west);
    }
    \node[conflictmark] at ($(b2.east)!0.5!(b3.west)$) {};
    \node[note, anchor=south] at ($(b2)!0.5!(b3)+(0,0.25)$) {conflict at $t^\star$};
  \end{tikzpicture}
};
\node[note, align=center] at (-0.8,7.55) {(i) Per-agent sequence $\pi_i$ \& MILP feasibility check};

\begin{scope}[yshift=3.9cm]
  \node[regionnode] (r1) at (-1.9,0.7) {$R_1$};
  \node[regionnode] (r2) at ( 0.0,0.7) {$R_2$};
  \node[regionnode] (r3) at ( 1.9,0.7) {$R_3$};
  \node[regionnode] (r5) at ( 0.0,-0.7) {$R_5$};

  \draw[edge] (r1) -- (r2);
  \draw[edge] (r2) -- (r5);
  \draw[edge] (r5) -- (r3);

  \draw[banned] (r2) -- (r3);
  \node[conflictmark] at ($(r2)!0.5!(r3)$) {};

  \node[note, align=center] at (-2.5,1.5) {(ii) Blacklist at $t^\star$};
  \node[note, align=center] at (1.7,-0.2) {$\text{ban }(R_2\!\to\!R_3)$};
\end{scope}

\node[seqbox] (seqB) at (0,1.4) {
  \begin{tikzpicture}[x=1cm,y=1cm, baseline=(current bounding box.center)]
    \foreach \t in {1,...,7} {\node at (\t,0.3) {$t=\t$};}
    \foreach \t/\r in {1/$R_1$,2/$R_2$,3/$R_5$,4/$R_3$,5/$R_3$,6/$R_2$,7/$R_2$} {
      \node[draw, rounded corners=1.5pt, fill=white, minimum width=7mm, minimum height=4mm] (c\t) at (\t,1) {\r};
    }
    \foreach \t in {1,...,6} {\draw[seqarrow] (c\t.east) -- (c\the\numexpr\t+1\relax.west);}
    \node[note, anchor=south] at ($(c2)!0.5!(c3)+(0,0.25)$) {detour};
  \end{tikzpicture}
};
\node[note, align=center] at (-2.3,2.5) {(iii) New $\pi_i$ \& resolve};

\draw[->, thick] (seqA.south) -- node[right] {infeasible} (0,5);
\draw[->, thick] (0,2.8) -- node[right] {replan} (seqB.north);

\end{tikzpicture}
\caption{Refinement on infeasibility. If the MILP is infeasible at \(t^\star\) for a candidate sequence, blacklist the offending \emph{time-indexed} transition in \(G_R\) (crossed edge), replan sequences to detour (e.g., via \(R_5\)), and resolve.}
\label{fig:refine_on_conflict}
\end{figure}
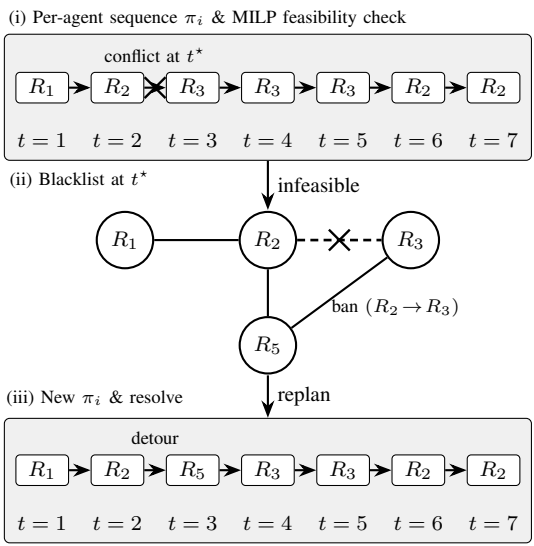

\begin{figure*}[t]
\centering
\begin{tikzpicture}[font=\footnotesize, >=Stealth]

\tikzset{
  nodeExp/.style={circle, draw, thick, fill=white, minimum size=5.5mm, inner sep=0pt},
  nodePrune/.style={circle, draw, thick, dashed, fill=gray!10, minimum size=5.5mm, inner sep=0pt},
  nodeInc/.style={circle, draw, very thick, fill=gray!30, minimum size=5.5mm, inner sep=0pt},
  cross/.style={draw, cross out, line width=0.9pt, minimum size=5.5mm, inner sep=0pt},
  edge/.style={line width=0.8pt},
  note/.style={font=\scriptsize}
}

\begin{scope}[xshift=0cm, yshift=0cm]
\node[note, anchor=west] at (-4,1.2) {(a) Naive MILP};

\coordinate (a0) at (0,1.0);
\coordinate (a1) at (-2.2,-0.2);
\coordinate (a2) at ( 2.2,-0.2);
\coordinate (a3) at (-3.2,-1.6);
\coordinate (a4) at (-1.2,-1.6);
\coordinate (a5) at ( 1.2,-1.6);
\coordinate (a6) at ( 3.2,-1.6);
\coordinate (a7) at (-3.8,-3.0);
\coordinate (a8) at (-2.6,-3.0);
\coordinate (a9) at (-1.8,-3.0);
\coordinate (a10) at (-0.6,-3.0);
\coordinate (a11) at (0.6,-3.0);
\coordinate (a12) at (1.8,-3.0);
\coordinate (a13) at (2.6,-3.0);
\coordinate (a14) at (3.8,-3.0);

\foreach \u/\v in {a0/a1,a0/a2,a1/a3,a1/a4,a2/a5,a2/a6,
                   a3/a7,a3/a8,a4/a9,a4/a10,a5/a11,a5/a12,a6/a13,a6/a14}{
  \draw[edge] (\u) -- (\v);
}

\node[nodeExp] at (a0) {};
\node[nodeExp] at (a1) {};
\node[nodeExp] at (a2) {};
\node[nodeExp] at (a3) {};
\node[nodePrune,label=left:{\scriptsize pruned}] at (a4) {};
\node[nodeExp] at (a5) {};
\node[nodeExp] at (a6) {};

\node[cross] at (a7) {};
\node[nodePrune] at (a8) {};
\node[nodeExp] at (a9) {};
\node[cross] at (a10) {};
\node[nodePrune] at (a11) {};
\node[nodeInc,label=above:{\scriptsize incumbent}] at (a12) {};
\node[cross] at (a13) {};
\node[nodePrune] at (a14) {};

\node[note, align=left, anchor=west] at (1.3,1.0cm)
{$b_{\text{naive}} \sim O(N^2 T L)$\\many collision binaries\\weak LP relaxation};
\end{scope}

\begin{scope}[xshift=9cm, yshift=-1cm]
\node[note, anchor=west] at (-3.6,2.2) {(b) Proposed formulation};

\coordinate (b0) at (0,1.0);
\coordinate (b1) at (-1.6,-0.3);
\coordinate (b2) at ( 1.6,-0.3);
\coordinate (b3) at (-2.2,-1.7);
\coordinate (b4) at (-1.0,-1.7);
\coordinate (b5) at ( 1.0,-1.7);
\coordinate (b6) at ( 2.2,-1.7);

\foreach \u/\v in {b0/b1,b0/b2,b1/b3,b1/b4,b2/b5,b2/b6}{
  \draw[edge] (\u) -- (\v);
}

\node[nodeExp,label=above:{\scriptsize root}] at (b0) {};
\node[nodePrune,label=left:{\scriptsize fathomed}] at (b1) {};
\node[nodeExp] at (b2) {};
\node[cross] at (b3) {};
\node[nodePrune] at (b4) {};
\node[nodeInc,label=above left:{\scriptsize incumbent}] at (b5) {};
\node[nodePrune] at (b6) {};

\node[note, anchor=west, align=left] at (0.8,1.6cm)
{$b_{\text{PAAMP}} \sim O(\rho N^2 T L)$ \\
with $\rho<1$\\
region membership fixed\\
relevant pairs only\\
tighter LP bounds};
\end{scope}
\end{tikzpicture}
\caption{Branch and bound comparison. Left: naive MILP with many active binaries and a large search tree. Right: PAAMP formulation with region membership and relevant pair pruning reduces binaries and tightens the LP relaxation, so many branches are fathomed early and the tree is smaller.}
\label{fig:bb_compare}
\end{figure*}
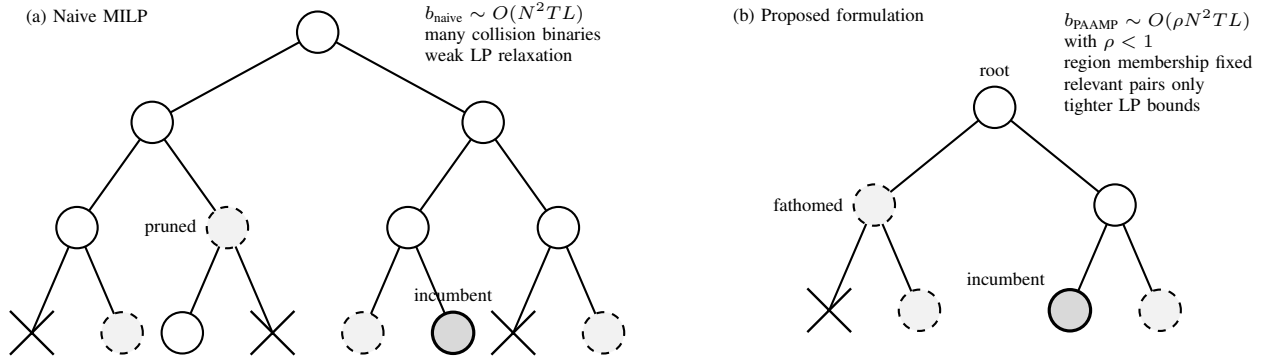

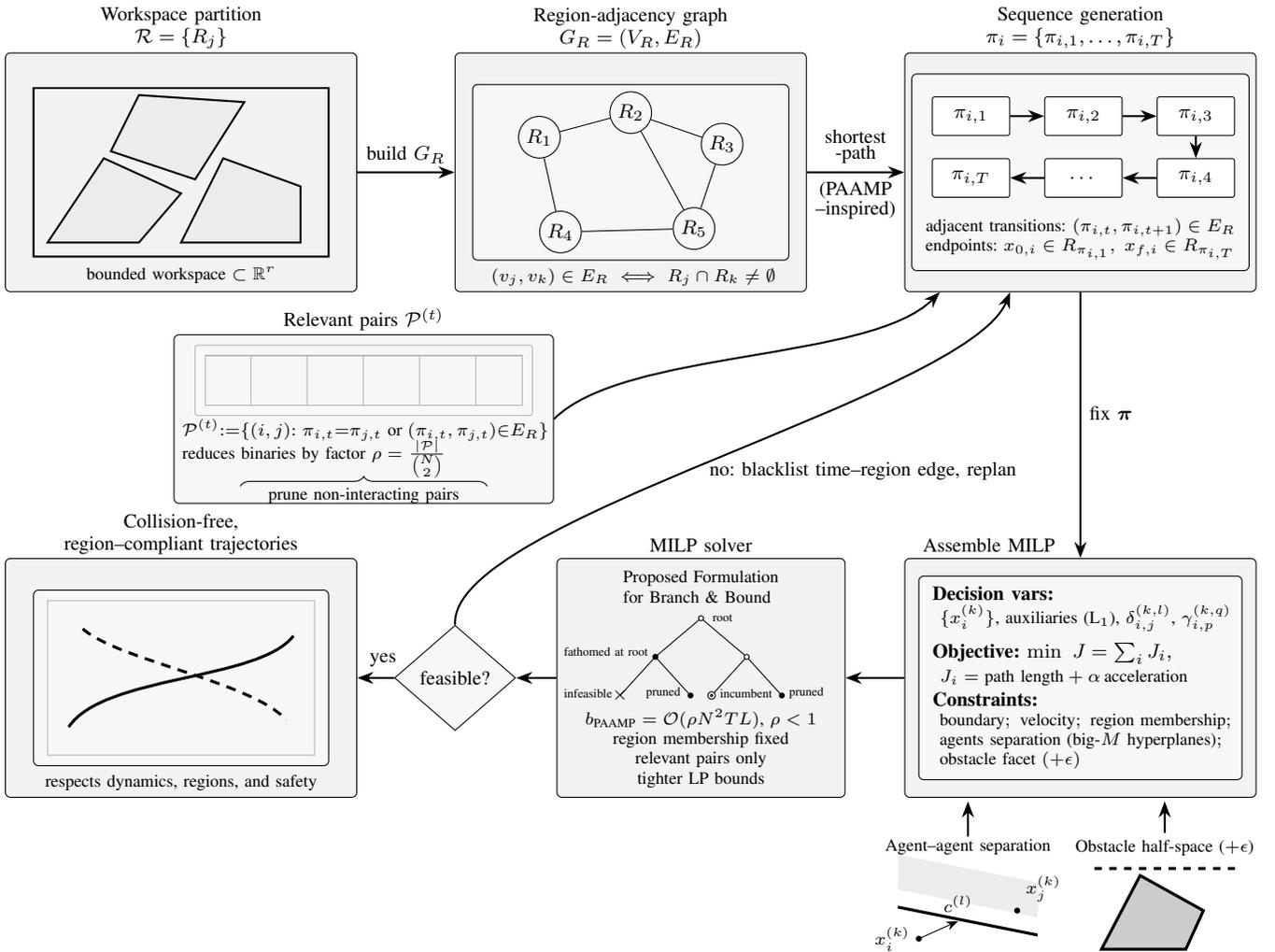
\begin{figure*}[b]
\centering
\begin{tikzpicture}[font=\footnotesize, >=Stealth]

\tikzset{
  block/.style   ={draw, rounded corners=2pt, fill=gray!10, minimum width=5.0cm, minimum height=3.4cm, align=center},
  small/.style   ={draw, rounded corners=2pt, fill=gray!06, inner sep=2pt},
  decision/.style={draw, diamond, aspect=2, inner sep=1.6pt, fill=gray!06},
  region/.style  ={draw, thick, fill=gray!15},
  obstacle/.style={draw, very thick, fill=gray!40},
  edgeA/.style   ={->, thick},
  note/.style    ={font=\scriptsize, align=left},
  brace/.style   ={decorate, decoration={brace, amplitude=4pt}}
}

\def\xgap{6.4cm}
\def\ygap{3.6cm}

\node[block, label={[align=center,label distance=-2pt]north:Workspace partition\\ $\mathcal R=\{R_j\}$}] (ws)  at (-\xgap,  \ygap) {};
\node[block, label={[align=center,label distance=-2pt]north:Region-adjacency graph\\ $G_R=(V_R,E_R)$}] (gr)  at ( 0cm,    \ygap) {};
\node[block, label={[align=center,label distance=-2pt]north:Sequence generation\\ $\pi_i=\{\pi_{i,1},\dots,\pi_{i,T}\}$}] (seq) at ( \xgap,  \ygap) {};

\node[block, label={[align=center,label distance=-1pt,xshift=-37pt]north:Assemble MILP}] (milp)   at (\xgap, -\ygap) {};
\node[block, minimum width=4.1cm, label={[align=center,label distance=-1pt]north:{MILP solver}}] (solveblk) at (1cm, -\ygap) {}; 
\node[decision, minimum size=1.5cm] (dec) at ($(solveblk.center)+(-3.5cm,0)$) {feasible?};
\node[block, label={[align=center,label distance=-2pt]north:Collision‐free,\\ region--compliant trajectories}] (out) at (-\xgap, -\ygap) {};

\draw[edgeA] (ws) -- node[above] {build $G_R$} (gr);
\draw[edgeA] (gr) -- node[above,yshift=8pt]{shortest} node[above]{‐path} node[below]{(PAAMP} node[below,yshift=-8pt]{--inspired)} (seq);

\draw[edgeA] (seq.south) -- node[right, yshift=5pt] {fix $\boldsymbol\pi$} (milp.north);

\draw[edgeA] (milp.west) -- node[above] {} (solveblk.east);
\draw[edgeA] (solveblk) -- (dec);
\draw[edgeA] (dec.west) -- node[above,xshift=3pt,yshift=1pt] {yes} (out.east);

\draw[edgeA] (dec.north) .. controls +(-0.1,1.5) and +(-0.6,-1.5) .. node[pos=0.65, below, yshift=-18.5pt, xshift=12pt] {no: blacklist time--region edge, replan} ($(seq.south)+(-1cm,0)$);

\begin{scope}
  \path let \p1=(ws.center) in
    coordinate (W0) at ($(ws.center)+(-2.1cm,-1.2cm)$)
    coordinate (W1) at ($(ws.center)+(+2.1cm,+1.2cm)$);
  \draw[thick] (W0) rectangle (W1);
  \path[region] ($(W0)+(0.2,0.2)$) --
                ($(W0)+(1.5,0.2)$) --
                ($(W0)+(2.1,0.9)$) --
                ($(W0)+(1.0,1.4)$) -- cycle;
  \path[region] ($(W0)+(2.1,0.2)$) --
                ($(W0)+(3.8,0.2)$) --
                ($(W0)+(3.8,1.0)$) --
                ($(W0)+(2.7,1.4)$) -- cycle;
  \path[region] ($(W0)+(1.1,1.5)$) --
                ($(W0)+(2.3,1.1)$) --
                ($(W0)+(3.0,2.2)$) --
                ($(W0)+(1.5,2.3)$) -- cycle;
  \node[note] at ($(ws.south)+(0,0.25)$) {bounded workspace $\subset\mathbb R^r$};
\end{scope}

\begin{scope}
  \node[small, anchor=center, minimum width=4.5cm, minimum height=2.5cm] (GRbox) at (gr.center) {};
  \node[circle, draw, inner sep=1.6pt] (r1) at ($(GRbox.center)+(-1.3,0.5)$) {$R_1$};
  \node[circle, draw, inner sep=1.6pt] (r2) at ($(GRbox.center)+(0,0.85)$) {$R_2$};
  \node[circle, draw, inner sep=1.6pt] (r3) at ($(GRbox.center)+(1.3,0.4)$) {$R_3$};
  \node[circle, draw, inner sep=1.6pt] (r4) at ($(GRbox.center)+(-1.0,-0.85)$) {$R_4$};
  \node[circle, draw, inner sep=1.6pt] (r5) at ($(GRbox.center)+(0.9,-0.8)$) {$R_5$};
  \foreach \u/\v in {r1/r2,r2/r3,r1/r4,r2/r5,r4/r5,r3/r5} {\draw (\u)--(\v);}
  \node[note, anchor=north] at ($(GRbox.south)+(0,0.02)$) {$\;(v_j,v_k)\in E_R \iff R_j\cap R_k\neq\emptyset$};
\end{scope}

\begin{scope}
    \node[small, anchor=center, minimum width=4.8cm, minimum height=2.8cm] (SEQbox) at (seq.center) {};

    \node[draw, rounded corners=1pt, fill=white, minimum width=11mm, minimum height=5.5mm] (b1) at ($(SEQbox.west)+(0.85,0.80)$) {$\pi_{i,1}$};
    \node[draw, rounded corners=1pt, fill=white, minimum width=11mm, minimum height=5.5mm] (b2) at ($(SEQbox.west)+(2.45,0.80)$) {$\pi_{i,2}$};
    \node[draw, rounded corners=1pt, fill=white, minimum width=11mm, minimum height=5.5mm] (b3) at ($(SEQbox.west)+(4.05,0.80)$) {$\pi_{i,3}$};
    
    \node[draw, rounded corners=1pt, fill=white, minimum width=11mm, minimum height=5.5mm] (b4) at ($(SEQbox.west)+(0.85,-0.08)$) {$\pi_{i,T}$};
    \node[draw, rounded corners=1pt, fill=white, minimum width=11mm, minimum height=5.5mm] (b5) at ($(SEQbox.west)+(2.45,-0.08)$) {$\cdots$};
    \node[draw, rounded corners=1pt, fill=white, minimum width=11mm, minimum height=5.5mm] (b6) at ($(SEQbox.west)+(4.05,-0.08)$) {$\pi_{i,4}$};
    
    \draw[edgeA] (b1.east) -- (b2.west);
    \draw[edgeA] (b2.east) -- (b3.west);
    \draw[edgeA] (b3.south) -- (b6.north);
    \draw[edgeA] (b6.west) -- (b5.east);
    \draw[edgeA] (b5.west) -- (b4.east);

    \node[note, align=center] at ($(SEQbox.center)+(0,-0.95)$)
    {adjacent transitions: $(\pi_{i,t},\pi_{i,t+1})\in E_R$\\
    endpoints: $x_{0,i}\in R_{\pi_{i,1}},\;x_{f,i}\in R_{\pi_{i,T}}$};
\end{scope}

\begin{scope}
\node[small, anchor=north, minimum width=5.4cm, minimum height=2.4cm,
      label={[align=center,label distance=-2pt]north:Relevant pairs $\mathcal P^{(t)}$}]
      (RP) at ($(seq.south)+(-10.2,-0.6)$) {};

\node[draw, rounded corners=1pt, inner sep=0pt, line width=0.3pt, draw=gray!60,
      minimum width=4.8cm, minimum height=1.0cm, anchor=north] (RPgrid)
      at ($(RP.north)+(0,-0.15cm)$) {};
      
\draw[gray!60, thin] ($(RPgrid.south west)+(0.15,0.15)$)
  rectangle ($(RPgrid.north east)+(-0.15,-0.15)$);
\foreach \t in {1,...,5} {
  \draw[gray!60, thin] ($(RPgrid.south west)+(0.80*\t,0.15)$)
    -- ($(RPgrid.north west)+(0.80*\t,-0.15)$);
}
\foreach \a in {1,...,1} {
  \draw[gray!60, thin] ($(RPgrid.west)+(0.15,0.15+0.20*\a)$)
    -- ($(RPgrid.east)+(-0.15,0.15+0.20*\a)$);
}

\node[note, align=left, anchor=north west] at ($(RPgrid.south west)+(-0.3,0.08)$)
{$\mathcal P^{(t)}{:=}\{(i,j){:}\;\pi_{i,t}{=}\pi_{j,t} \;\text{or}\;(\pi_{i,t},\pi_{j,t}){\in}E_R\}$\\
reduces binaries by factor $\rho=\tfrac{|\mathcal P|}{\binom{N}{2}}$};

\draw[brace] ($(RP.south west)+(0.95,0.25)$) -- node[below=-3pt] {\scriptsize prune non‐interacting pairs} ($(RP.south east)+(-0.95,0.25)$);
\draw[edgeA] (RP.east) .. controls +(1,1.0) and +(-1,-1.0) .. ($(seq.south)+(-2cm,0)$);
\end{scope}

\begin{scope}
  \node[small, anchor=center, minimum width=4.6cm, minimum height=2.9cm] (Mbox) at (milp.center) {};
  \node[anchor=west] at ($(Mbox.west)+(0.08,1.2)$) {\textbf{Decision vars:}};
  \node[note, anchor=west] at ($(Mbox.west)+(0.18,0.85)$) {$\{x_i^{(k)}\}$, auxiliaries (L$_1$), $\delta_{i,j}^{(k,l)}$, $\gamma_{i,p}^{(k,q)}$};
  
  \node[anchor=west] at ($(Mbox.west)+(0.08,0.35)$) {\textbf{Objective:} $\min\ J=\sum_i J_i,$};
  \node[note, anchor=west] at ($(Mbox.west)+(0.18,0.0)$)
  {$J_i=\text{path length}+\alpha\,\text{acceleration}$};

  \node[anchor=west] at ($(Mbox.west)+(0.08,-0.3)$) {\textbf{Constraints:}};
  \node[note, align=left, anchor=west] at ($(Mbox.west)+(0.18,-0.9)$)
  {$\text{boundary};\ \text{velocity};\ \text{region membership};$\\
   $\text{agents separation (big-}M\text{ hyperplanes)};$ \\$ \text{obstacle facet }(+\epsilon)$};
\end{scope}

\begin{scope}
  \node[small, anchor=north, minimum width=2cm, minimum height=0.5cm, white]
        (COL) at ($(milp.south)+(-1.6,-1.65)$) {};
  \coordinate (H1) at ($(COL.center)+(-1.00,0.35)$);
  \coordinate (H2) at ($(COL.center)+( 1.00,-0.05)$);
  \fill[gray!15] ($(H1)+(0,0.25)$) -- ($(H2)+(0,0.25)$) -- ($(H2)+(0,0.75)$) -- ($(H1)+(0,0.75)$) -- cycle;
  \draw[very thick] (H1) -- (H2);
  \coordinate (xi) at ($(COL.center)+(-0.70,-0.10)$);
  \coordinate (xj) at ($(COL.center)+( 0.70, 0.30)$);
  \fill (xi) circle (1.2pt) node[left] {\scriptsize $x_i^{(k)}$};
  \fill (xj) circle (1.2pt) node[above right] {\scriptsize $x_j^{(k)}$};
  \draw[->, thin] (xi) -- ++(0.56,0.25) node[above] {\scriptsize $c^{(l)}$};
  \node at ($(COL.north)+(0,0.95)$) {\scriptsize Agent--agent separation};
  \draw[edgeA] ($(COL.north)+(0,1.1)$)
    -- ($(milp.south)+(-1.6,-0.05)$);
\end{scope}

\begin{scope}
  \node[small, anchor=north, minimum width=3.6cm, minimum height=0.5cm, white]
        (OBS) at ($(milp.south)+(1.2,-1.55)$) {};
  \path[obstacle] ($(OBS.center)+(-0.90,-0.35)$) -- ++(1.20,0) -- ++(0.25,0.55) -- ++(-0.90,0.5) -- cycle;
  \draw[very thick, dashed] ($(OBS.center)+(-1.00,0.8)$) -- ($(OBS.center)+(1.00,0.8)$);
  \node at ($(OBS.north)+(0,0.84)$) {\scriptsize Obstacle half‐space ($+\epsilon$)};
  \draw[edgeA] ($(OBS.north)+(0,1.05)$) -- ($(milp.south)+(1.2,-0.05)$);
\end{scope}

\begin{scope}
\node[note, anchor=center] at ($(solveblk.center)+(0.0,1.3)$) {Proposed Formulation \\ for Branch \& Bound};

\coordinate (b0) at ($(solveblk.center)+(0,0.85)$);
\coordinate (b1) at ($(b0)+(-0.65,-0.55)$);
\coordinate (b2) at ($(b0)+( 0.65,-0.55)$);
\coordinate (b3) at ($(b1)+(-0.5,-0.55)$);
\coordinate (b4) at ($(b1)+( 0.5,-0.55)$);
\coordinate (b5) at ($(b2)+(-0.5,-0.55)$);
\coordinate (b6) at ($(b2)+( 0.5,-0.55)$);

\foreach \u/\v in {b0/b1,b0/b2,b1/b3,b1/b4,b2/b5,b2/b6} {\draw (\u)--(\v);}

\node[circle, draw, inner sep=0.7pt, fill=white, label=right:{\tiny root}] at (b0) {};
\node[circle, draw, inner sep=0.7pt, fill=black] at (b1) {};
\node[xshift=-20pt,yshift=2pt] at (b1) {\tiny fathomed at root};
\node[circle, draw, inner sep=0.7pt, fill=white] at (b2) {};
\node at (b3) {\scriptsize$\times$};
\node[xshift=-13pt,yshift=1pt] at (b3) {\tiny infeasible};
\node[circle, draw, inner sep=0.7pt, fill=black] at (b4) {};
\node[xshift=-11pt,yshift=1pt] at (b4) {\tiny pruned};
\node[circle, draw, inner sep=0.7pt, double] at (b5) {};
\node[xshift=14pt,yshift=1pt] at (b5) {\tiny incumbent};
\node[circle, draw, inner sep=0.7pt, fill=black] at (b6) {};
\node[xshift=10pt,yshift=1pt] at (b6) {\tiny pruned};

\node[note, align=center, anchor=center]
  at ($(solveblk.center)+(0,-1)$)
  {\scriptsize $b_{\text{PAAMP}}=\mathcal{O}(\rho N^2TL)$, $\rho<1$\\
   \scriptsize region membership fixed\\
   \scriptsize relevant pairs only\\
   \scriptsize tighter LP bounds};
\end{scope}

\begin{scope}
  \node[small, anchor=center, minimum width=4.2cm, minimum height=2.5cm] (OUTin) at (out.center) {};
  \draw[gray!60] ($(OUTin.center)+(-1.9,-1.1)$) rectangle ($(OUTin.center)+(1.9,1.1)$);
  \draw[very thick] ($(OUTin.center)+(-1.6,-0.7)$) .. controls +(0.5,0.75) and +(-0.5,-0.65) .. ($(OUTin.center)+(1.6,0.6)$);
  \draw[very thick, dash pattern=on 4pt off 3pt] ($(OUTin.center)+(-1.4,0.7)$) .. controls +(0.5,-0.45) and +(-0.5,0.55) .. ($(OUTin.center)+(1.5,-0.6)$);
  \node[note, align=center] at ($(out.south)+(0,0.21)$) {respects dynamics, regions, and safety};
\end{scope}
\end{tikzpicture}
\caption{Overall methodology. The workspace is partitioned into convex regions to form a region-adjacency graph. For each agent, a region sequence $\pi_i$ is generated (sequence-then-solve). Relevant‐pair pruning builds the set $\mathcal P^{(t)}$ of agent pairs to constrain at each time, reducing binary variables by $\rho$. The MILP assembles an $L_1$ path-length plus acceleration objective with boundary, dynamics, region membership, agent-agent separating‐hyperplane, and obstacle facet constraints. A MILP solver returns feasible, collision‐free trajectories; if infeasible, the offending time-region edge is blacklisted and sequences are replanned.}
\label{fig:overall_method}
\end{figure*}

\end{document}